\documentclass[lettersize,journal]{IEEEtran}
\pdfoutput=1
\usepackage{amsmath,amsfonts}
\usepackage{algorithmic}
\usepackage{algorithm}
\usepackage{array}
\usepackage{subcaption}
\usepackage{amsthm}
\usepackage{mathrsfs}
\usepackage{colortbl}
\usepackage{hhline}
\usepackage{caption}
\usepackage{newfloat}
\usepackage{listings}
\usepackage{adjustbox}
\usepackage{cleveref}
\usepackage[caption=false,font=normalsize,labelfont=sf,textfont=sf]{subfig}
\usepackage{textcomp}
\usepackage{stfloats}
\usepackage{booktabs}
\usepackage{url}
\usepackage{verbatim}
\usepackage{graphicx}
\usepackage{cite}
\usepackage{amssymb}
\usepackage{multirow}
\usepackage{tablefootnote}
\usepackage[dvipsnames]{xcolor}
\usepackage{tabularx}
\usepackage{tikz}
\usetikzlibrary{tikzmark, positioning, arrows.meta}
\definecolor{Ocean}{RGB}{129,194,234}

\newtheorem{theorem}{Theorem}

\begin{document}

\title{A Fresh Look at Generalized Category Discovery through Non-negative Matrix Factorization}

\author{Zhong~Ji, \textit{Senior Member, IEEE}, Shuo Yang, Jingren~Liu, Yanwei~Pang, \textit{Senior Member, IEEE}, Jungong Han, \textit{Senior Member, IEEE} 
\thanks{This work was supported by the National Key Research and Development Program of China (Grant No. 2022ZD0160403), and the National Natural Science Foundation of China (NSFC) under Grant 62176178 (Corresponding author: Jingren~Liu).}
\thanks{Zhong~Ji and Yanwei~Pang are with the School of Electrical and Information Engineering, Tianjin Key Laboratory of Brain-Inspired Intelligence Technology, Tianjin University, Tianjin 300072, China, and also with the Shanghai Artificial Intelligence Laboratory, Shanghai 200232, China (e-mail: \{jizhong, pyw\}@tju.edu.cn).}
\thanks{Shuo Yang and Jingren Liu are with the School of Electrical and Information Engineering, Tianjin Key Laboratory of Brain-Inspired Intelligence Technology, Tianjin University, Tianjin 300072, China (e-mail: \{yang\_shuo, jrl0219\}@tju.edu.cn).}
\thanks{Jungong Han is with the Department of Computer Science, the University of Sheffield, UK (e-mail: jungonghan77@gmail.com).}
}

\markboth{Journal of \LaTeX\ Class Files,~Vol.~14, No.~8, August~2021}%
{Shell \MakeLowercase{\textit{et al.}}: A Sample Article Using IEEEtran.cls for IEEE Journals}


\maketitle

\begin{abstract}
Generalized Category Discovery (GCD) aims to classify both base and novel images using labeled base data. However, current approaches inadequately address the intrinsic optimization of the co-occurrence matrix $\bar{A}$ based on cosine similarity, failing to achieve zero base-novel regions and adequate sparsity in base and novel domains. To address these deficiencies, we propose a Non-Negative Generalized Category Discovery (NN-GCD) framework. It employs Symmetric Non-negative Matrix Factorization (SNMF) as a mathematical medium to prove the equivalence of optimal K-means with optimal SNMF, and the equivalence of SNMF solver with non-negative contrastive learning (NCL) optimization. Utilizing these theoretical equivalences, it reframes the optimization of $\bar{A}$ and K-means clustering as an NCL optimization problem. Moreover, to satisfy the non-negative constraints and make a GCD model converge to a near-optimal region, we propose a GELU activation function and an NMF NCE loss. To transition $\bar{A}$ from a suboptimal state to the desired $\bar{A}^*$, we introduce a hybrid sparse regularization approach to impose sparsity constraints. Experimental results show NN-GCD outperforms state-of-the-art methods on GCD benchmarks, achieving an average accuracy of 66.1\% on the Semantic Shift Benchmark, surpassing prior counterparts by 4.7\%.
\end{abstract}


\section{Introduction}

\begin{figure}[h] 
    \centering
    \begin{subfigure}[b]{0.45\textwidth}
        \centering
        \includegraphics[width=\textwidth]{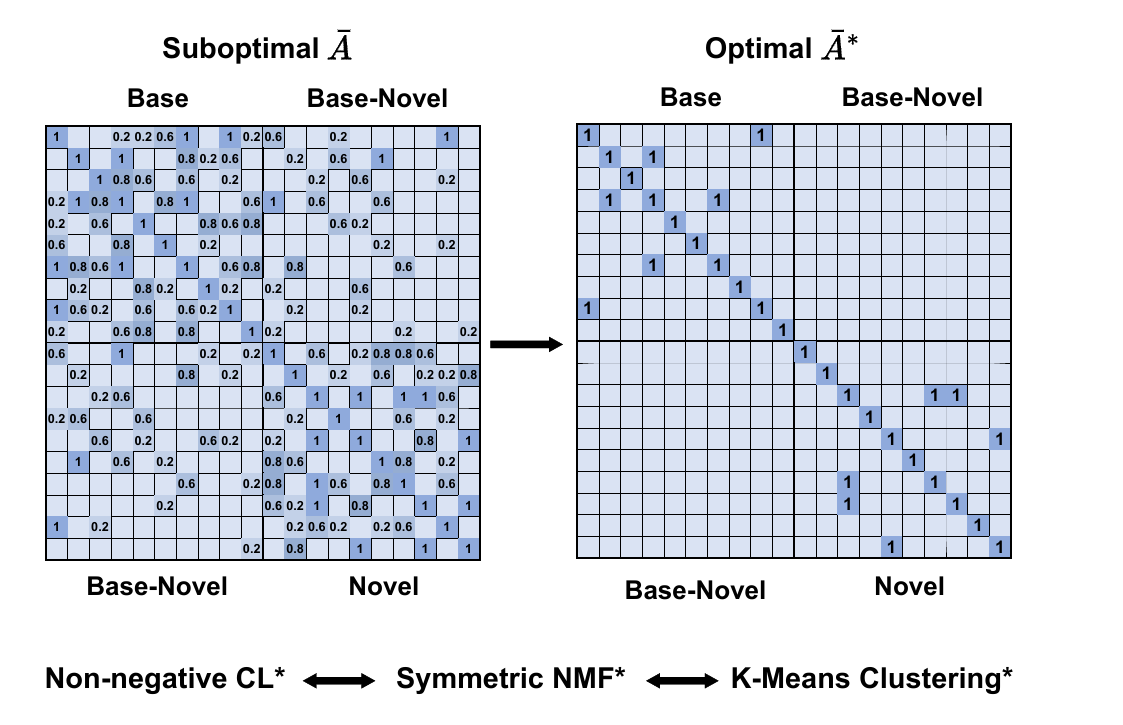}
    \end{subfigure}
    \caption{The comparison of a suboptimal co-occurrence matrix \(\bar{A}\) and its optimally refined counterpart \(\bar{A}^*\), alongside the theoretical underpinnings involved in attaining the optimal \(\bar{A}^*\) in GCD scenarios, including their equivalence. The asterisk \(*\) symbol denotes that the components have been optimized to their respective optimal conditions.}
    \label{fig:main_motivation}
\end{figure}

In recent years, deep learning models \cite{he2016deep, dosovitskiy2020image, caron2021emerging, chan2015pcanet, jin2017deep, zhang2017beyond} have made substantial strides, primarily attributed to the proliferation of large-scale, meticulously annotated datasets. However, these models fundamentally operate under a ``closed-world" assumption, limiting their efficacy in real-world scenarios where data typically encompasses a blend of seen and unseen categories \cite{han2019learning,han2021autonovel,fini2021unified}. To address this limitation, the field of Generalized Category Discovery (GCD) \cite{vaze2022generalized, wen2023parametric} has emerged, focusing on the development of methodologies capable of recognizing both seen and unseen categories utilizing partially labeled data, while eschewing reliance on additional modal information. Despite significant advancements, current GCD approaches \cite{vaze2024no, wang2024sptnet,rastegar2024learn} continue to confront formidable challenges. A particularly salient issue is the suboptimal class-level and domain-level separability achieved when fine-tuning self-supervised pre-trained models through Contrastive Learning(CL). This separability lacks the necessary clarity and orderliness to adequately reflect the sparsity of intra-class similarity within both the base and novel domains, as well as the non-confounding nature within base-novel interaction domains.

To further elucidate, we construct the normalized co-occurrence matrix \(\bar{A}\) in Fig.~\ref{fig:main_motivation} based on the optimized feature space to articulate the cosine similarity between base and novel samples. Previous approaches \cite{chiaroni2023parametric, zhang2023promptcal}, which merely constrain the feature space through contrastive learning, typically achieve a similarity matrix akin to that shown in the top-left corner of Fig.~\ref{fig:main_motivation}. Although these methods yield commendable GCD performance, they fall short of perfection. 
In the GCD scenario, if a model is optimized to its theoretical optimum, an ideal similarity matrix \(\bar{A}^*\), as illustrated in the upper right corner of Fig.~\ref{fig:main_motivation}, must satisfy two key criteria: (1) \textbf{Intra-class Sparsity:} Given that each class constitutes only a small portion of the data in its respective domain, the submatrix regions corresponding to the base and novel domains must exhibit sufficient sparsity. (2) \textbf{Base-Novel Insulation:} The interaction regions between base and novel domains must demonstrate significant dissimilarity, with similarities approaching zero, as only with substantial differences can a GCD model accurately identify them during clustering.

To induce \(\bar{A}\) to possess the desired properties, we focus on the pivotal component among various GCD methods: K-means clustering. As our starting point, K-means clustering performs two-dimensional manifold compression and differentiation on the output feature space. In the GCD scenario, the performance of K-means clustering and the evolution of \(\bar{A}\) are equivalent. During the evolution of \(\bar{A}\) to \(\bar{A}^*\), K-means clustering will also undergo a concurrent transformation, exhibiting a pronounced separation between the base and novel domains and significantly enhancing the class-level separability within each domain. Consequently, a GCD model that produces an optimal K-means clustering can obtain \(\bar{A}^*\).
However, how can this optimal K-means clustering be achieved? Through meticulous mathematical analysis, we have derived two pivotal theorems, denoted as Theorem~\ref{main_theorem1} and Theorem~\ref{main_theorem3}, which prove that in the context of GCD: (1) An optimal K-means clustering can be reformulated as an optimal Symmetric Non-negative Matrix Factorization (SNMF) solver. (2) During optimization, the SNMF solver can be equivalently formulated as a Non-negative Contrastive Learning (NCL) optimization. Under the guidance of these theorems, the evolution of \(\bar{A}\) is transformed into an NCL optimization problem, enabling a GCD model to obtain \(\bar{A}^*\) by minimizing the NCL loss.

Under the guidance of these theorems, we propose the Non-Negative Generalized Category Discovery (NN-GCD) framework. Particularly, to satisfy the basic paradigm of NCL, we introduce the Non-negative Activated Neurons mechanism, using the GELU activation function to alleviate the dead neurons problem. Subsequently, we develop the NMF NCE loss to further optimize the NCL process, enabling the GCD model to converge to a globally optimal region. However, gradient descent generally cannot converge to the theoretical global optimum, meaning \(\bar{A}\) cannot achieve perfect sparsity. Therefore, we propose a Hybrid Sparse Regularization approach, evolving \(\bar{A}\) to \(\bar{A}^*\), thereby achieving sparsity and improving intra-class sparsity and base-novel category insulation.

Our contributions encapsulate the following three pivotal aspects: 1) \textbf{Theoretical Foundation}: We establish a rigorous theoretical equivalence among K-means clustering, SNMF, and NCL. This foundational insight delineates the precise trajectory GCD tasks should pursue—attaining the optimal \(\bar{A}^*\) and the optimal optimization of NCL, thereby solidifying the theoretical underpinnings for our framework. 2) \textbf{NCL Optimization}: Tailored specifically for the NCL optimization within the GCD context, we introduce Non-negative Activated Neurons, NMF NCE loss, and Hybrid Sparse Regularization. These innovations, strategically integrated at both the network architecture and loss function levels, empower GCD models to achieve superior \(\bar{A}\) and novel category discovery capabilities. 3) \textbf{Empirical Validation}: Through comprehensive experimental evaluations and meticulous analyses, we substantiate the efficacy of our NN-GCD framework. Demonstrating unparalleled performance on seven established GCD benchmarks, our framework yields an enhanced \(\bar{A}\) matrix and underscores its practical utility and dominance in real-world applications.

\section{Related Works}
\textbf{Symmetric Non-negative Matrix Factorization (SNMF)} has emerged as a crucial extension of traditional Non-negative Matrix Factorization (NMF) \cite{lee1999learning, zhang2008topology, guan2011manifold, lu2018structurally}, specifically tailored for scenarios where the data matrix is symmetric and non-negative, such as in similarity or adjacency matrices. SNMF decomposes a high-dimensional symmetric matrix into a product of lower-rank, symmetric, and non-negative matrices, thereby enhancing interpretability and robustness in capturing latent structures. The significance of SNMF has been demonstrated across various domains. In clustering \cite{berahmand2024wsnmf, kuang2012symmetric}, SNMF outperforms traditional methods by providing a more nuanced and interpretable decomposition of data relationships. It has also proven effective in data representation \cite{cai2010graph, zong2017multi}, where compact and interpretable representations are essential for downstream tasks. Additionally, SNMF excels in feature selection \cite{guan2011manifold, liu2010non, jing2012snmfca}, enabling the identification of discriminative features under non-negativity constraints, which is crucial for maintaining the interpretability and relevance of the selected features. Furthermore, SNMF has found significant applications in community detection \cite{luo2020highly, luo2021symmetric}, where its symmetric nature is particularly advantageous for identifying densely connected sub-networks within large graphs. These applications highlight SNMF's ability to enhance both the theoretical understanding and practical utility of non-negative matrix factorization techniques, making it a powerful tool in a variety of data-driven fields.

\textbf{Generalized Category Discovery (GCD)}, an extension of the Novel Category Discovery (NCD) paradigm \cite{han2019learning, fini2021unified, han2021autonovel}, addresses the challenge of discovering novel, previously unseen categories while leveraging labeled data from seen categories. GCD aligns with open-world semi-supervised learning \cite{cao2021open, sun2023opencon} by jointly handling both seen and unseen categories in a dynamic, open environment. Current GCD approaches can be broadly classified into two categories: non-parametric and parametric methods. Non-parametric methods \cite{vaze2022generalized, pu2023dynamic, zhang2023promptcal} typically employ clustering techniques, such as K-means, to separate unseen classes, while relying on minimal assumptions about the data distribution. In contrast, parametric methods \cite{chiaroni2023parametric, vaze2024no, wang2024sptnet} incorporate implicit clustering within learned classifiers, enabling more sophisticated modeling of category discovery by learning discriminative features for unseen categories. Despite the widespread adoption of clustering in both explicit and implicit GCD approaches, existing methods have not fully addressed the specific optimization goals of GCD, particularly concerning the nature of \(\bar{A}\) and how K-means clustering should be optimized, thus limiting the full potential. Our NN-GCD framework addresses this limitation both theoretically and practically by transforming K-means clustering into an optimizable NCL problem through SNMF, which provides a more cohesive and effective strategy.

\section{Preliminaries}
In this section, we provide a detailed introduction to the media involved in proving the equivalence of optimal NCL and \(\bar{A}^*\). This exposition primarily focuses on the foundational concepts of Non-negative Matrix Factorization (NMF) and its symmetric variant, SNMF.

\textbf{Non-negative Matrix Factorization} complements clustering by decomposing a non-negative matrix \( V \in \mathbb{R}^{m \times n}_+ \) into the product of two lower-rank non-negative matrices \( W \in \mathbb{R}^{m \times k}_+ \) and \( H \in \mathbb{R}^{k \times n}_+ \), where typically \( k < \min(m,n) \) \cite{lee1999learning}. The canonical NMF optimization problem is formulated as:
\begin{equation}
    \min_{W \geq 0, H \geq 0} \| V - WH \|_F^2,
\end{equation}
where $\| \cdot \|_F^2$ represents the Frobenius norm.

\textbf{Symmetric Non-negative Matrix Factorization} extends NMF to symmetric matrices where \( V = V^T \) \cite{ding2006nonnegative}. SNMF approximates \( V \) as \( HH^T \), with \( H \) being non-negative:
\begin{equation}
    \min_{H \geq 0} \| V - HH^T \|_F^2.
\end{equation}

This formulation inherently aligns with the intrinsic clustering nature of GCD tasks, preserving data symmetry and maintaining inherent relationships. As GCD methodologies advance, integrating SNMF with high-dimensional, multi-modal data presents promising research directions. This synergy enhances our capacity to uncover and interpret complex patterns across diverse domains. SNMF's application in GCD contexts may revolutionize class discovery, particularly where traditional clustering methods face limitations due to data complexity or high dimensionality.

\textbf{Non-negative Contrastive Learning} is a variant of contrastive learning that imposes non-negativity constraints on the network outputs during the optimization process \cite{wang2024non}. The NCL loss function is defined as:
\begin{multline}
\mathcal{L}_{\mathrm{NCL}} (f) = -2 \mathbb{E}_{x, x^{+}} \left[ f_{+}(x)^\top f_{+}(x^{+}) \right] \\
+ \mathbb{E}_{x, x^{-}} \left[ f_{+}(x)^\top f_{+}(x^{-}) \right]^2,
\end{multline}
where \( f_{+}(x) \) denotes the non-negative output of the neural network for input \( x \). 

The first term encourages similar representations for positive pairs, promoting intra-class compactness, while the second term penalizes similar representations for negative pairs, ensuring inter-class separability. NCL aims for sparse and disentangled representations, enhancing the model's generalization ability, which aligns with the requirements of the GCD tasks for generalizing from old to new classes and disentangling features between them.

\section{Theoretical Insights}
In this section, to establish the theoretical foundation for the equivalence among optimal K-means clustering, SNMF, and NCL, we derive Theorem~\ref{main_theorem1} and Theorem~\ref{main_theorem3}.

\subsection{Equivalence of Kernel K-Means and SNMF}
\begin{theorem}[Equivalence of optimal Kernel K-Means \footnote{Applying K-means to the inner products of features constitutes a form of Kernel K-means \cite{guo2021deep}.} and optimal SNMF] \label{main_theorem1}
Let \(\{x_i\}_{i=1}^n\) be a dataset transformed into a Reproducing Kernel Hilbert Space (RKHS) via an optimal neural network $f^*$, resulting in optimal kernel matrix \(A^* = f^*(X)^\top f^*(X)\), which is unnormalized \(\bar{A^*}\). The optimal Kernel K-means objective can be reformulated as:
\begin{equation}
    \small
    \begin{aligned}
       \min \mathcal{L}_{A^*} &= \min(\sum_i \|f^*(x_i)\|^2 - \sum_{k} \sum_{i, j \in C_k}h_k^\top f^*(x_i)^\top f^*(x_j)h_k) \\
        &= \min \operatorname{Tr}(A^{*\top} A^*) - \min \operatorname{Tr}(H^\top A^* H), \\
    \end{aligned}
\end{equation}
where \( H = (h_1, \ldots, h_K) \) with \( h_k^\top h_l = \delta_{kl} \), \( H^T H = I \), and \( A^{*\top} A^* = \text{const} \).

Expanding the optimal kernel k-means loss:
\begin{equation}
    \begin{aligned}
       \min \mathcal{L}_{A^*} &= \text{const} - \min \operatorname{Tr}(H^\top A^* H) \\
       &= \min(\|A^{*}\|^2-2 \operatorname{Tr}(H^\top A^* H) + \|H^T H\|^2). \\
    \end{aligned}
\end{equation}

This is equivalent to an instance of optimal SNMF:
\begin{equation}
\min_{H^\top H = I, H \geq 0} \| A^* - H H^\top \|^2.
\end{equation}
\end{theorem}

\begin{proof}
To demonstrate the equivalence between NMF and K-means clustering in a kernelized setting, consider a dataset $\{x_i\}_{i=1}^n$ mapped into a Reproducing Kernel Hilbert Space (RKHS) through a transformation $f(x)$. The optimization problem for Kernel K-means in this space is expressed as:
\begin{equation}
    \small
    \min \mathcal{L}_{K}(f) = \sum_i \|f(x_i)\|^2 - \sum_{k} \frac{1}{n_k} \sum_{i, j \in C_k} f(x_i)^\top f(x_j),
\end{equation}
where the term $\sum_i \|f(x_i)\|^2$ remains constant for a fixed network architecture and can thus be omitted from the optimization process.

The clustering solution is encoded by the matrix $H$ consisting of $K$ indicator vectors, where:
\begin{equation}
    H = (h_1, \ldots, h_K), \quad h_k^\top h_l = \delta_{kl},
\end{equation}
and each vector $h_k$ is normalized as:
\begin{equation}
    h_k = \left(0, \ldots, 0, \underbrace{1, \ldots, 1}_{n_k \text{}}, 0, \ldots, 0\right)^\top / \sqrt{n_k}.
\end{equation}
This formulation leads to the Kernel K-means objective being rewritten as:
\begin{equation}
    \mathcal{L}_A^* = \operatorname{Tr}(A^{*\top} A^*) - \operatorname{Tr}(H^\top A^* H),
\end{equation}
where \( A^* = f^*(X)^\top f^*(X) \) and \( f^*(X) = [f^*(x_1), \ldots, f^*(x_n)] \).

Since the term $\operatorname{Tr}(A^{*\top} A^*)$ is constant for a fixed network, the objective is to maximize $\operatorname{Tr}(H^\top A^* H)$ subject to $H^\top H = I$ and $H \geq 0$. This is equivalent to solving:
\begin{equation}
    \min_{H^\top H = I, H \geq 0} \| A^* - H H^\top \|^2.
\end{equation}

This problem corresponds to Symmetric Non-negative Matrix Factorization (SNMF), where a non-negative matrix $V \geq 0$ is factorized as $V \approx FF^\top$, with $F \geq 0$. In the context of SNMF, $A^*$ represents the data matrix, and $H H^\top$ is the projection onto the cluster centroids. The orthogonality constraint $H^\top H = I$ enforces the orthogonality of these centroids.

While SNMF may not exactly hold the equivalence to K-means in suboptimal cases, under optimal conditions, SNMF approaches the orthogonality properties required for K-means clustering, leading to their approximate equivalence.
\end{proof}

At this juncture, with the assistance of Theorem~\ref{main_theorem1}, the optimal K-means clustering, which is challenging to optimize during the test phase, can be interpreted as an SNMF problem associated with the output feature space during training. This insight further enables us to transform it into an optimizable form, facilitating the optimization of \(\bar{A}\).

\subsection{Equivalence of SNMF and NCL}
Following prior work \cite{haochen2021provable, wang2024non}, we establish the equivalence between SNMF on \(\bar{A}\) and NCL. Let \(\bar{A} = D^{-1/2} A D^{-1/2} \in \mathbb{R}^{N \times N}_+\) be the normalized adjacency matrix, where \(A\) is the co-occurrence matrix of augmented samples \(x \in X\), with entries \(A_{x,x'} := P(x, x') = \mathbb{E}_{\bar{x}}[\mathcal{A}(x|\bar{x}) \mathcal{A}(x'|\bar{x})]\) representing the co-occurrence probability of samples \(x\) and \(x'\). \footnote{The normalized augmented co-occurrence matrix, formed through data augmentation in contrastive learning, retains the same structure as the original matrix. Hence, both are represented by \(\bar{A}\).} Building on this foundation, we propose Theorem~\ref{main_theorem3}.

\begin{theorem} \label{main_theorem3}
Given a normalized adjacency matrix $\bar{A} \in \mathbb{R}^{N \times N}_+$ representing the probabilities of co-occurrence of augmented data samples, and a representation function $f: \mathbb{R}^d \rightarrow \mathbb{R}^k$ subject to non-negativity constraints $f_+(x) \geq 0$ for all $x \in X$, performing SNMF on $\bar{A}$ is equivalent to the NCL. Specifically, the SNMF problem minimizes the loss:
\begin{equation}
\begin{split}
\mathcal{L}_{\mathrm{SNMF}}(F) &= \left\| \overline{A} - F_+ F^\top_+ \right\|^2 \\
&= \sum_{x, x^{+}} \left( \frac{P(x, x^{+})^2}{P(x)P(x^{+})} \right) \Bigg\vert_{\text{const}} \\
& \quad -2 \mathbb{E}_{x, x^{+}} \left[ f_{+}(x)^\top f_{+}(x^{+}) \right] \\
& \quad + \mathbb{E}_{x, x^{-}} \left[ f_{+}(x)^\top f_{+}(x^{-}) \right]^2\\
&= \mathcal{L}_{\mathrm{NCL}} + \text{const},
\end{split}
\end{equation}
where $F_+\geq0$ and each row vector $(F_{+})_{x,:} = \sqrt{P(x)} f_{+}(x)^\top$. 

Therefore, minimizing $\mathcal{L}_{\mathrm{SNMF}}(F)$ is equivalent to minimizing $\mathcal{L}_{\mathrm{NCL}}(f)$ up to an additive constant.
\end{theorem}

\begin{proof}
Representation learning aims to train an encoder function $f: \mathbb{R}^d \rightarrow \mathbb{R}^k$ that maps input data $x \in \mathbb{R}^d$ to compact representations $z \in \mathbb{R}^k$. In contrastive learning, positive sample pairs $(x, x^+)$ are generated by augmenting the same sample $\bar{x} \sim \mathcal{P}(\bar{x})$, where augmentation follows the distribution $\mathcal{A}(\cdot|\bar{x})$. Negative samples $x^-$ are drawn independently from the marginal distribution $\mathcal{P}(x)$.

We define the normalized adjacency matrix \(\bar{A} = D^{-1/2} A D^{-1/2}\), where \(A \in \mathbb{R}^{N \times N}_+\) represents the co-occurrence matrix of augmented samples \(x \in X\), with entries \(A_{x,x'} := P(x, x') = \mathbb{E}_{\bar{x}}[\mathcal{A}(x|\bar{x}) \mathcal{A}(x'|\bar{x})]\). The objective of contrastive learning is to align the features of positive pairs while ensuring the separation of features for negative pairs.

The InfoNCE loss is defined as:
\begin{equation}
    \footnotesize
    \begin{split}
    &\mathcal{L}_{\mathrm{NCE}}(f) = -\mathbb{E}_{x, x_{+},\left\{x_{i}^{-}\right\}_{i=1}^{M}} \left( \log \left( \exp \left(f(x)^{\top} f\left(x_{+}\right)\right)\right) \right. \\
    &\quad \left. - \log \left( \exp \left(f(x)^{\top} f\left(x_{+}\right)\right) + \sum_{i=1}^{M} \exp \left(f(x)^{\top} f\left(x_{i}^{-}\right)\right) \right) \right).
    \end{split}
\end{equation}

For theoretical analysis, we use the spectral contrastive loss:
\begin{equation}
    \small
    \mathcal{L}_{\mathrm{sp}}(f) = -2\mathbb{E}_{x,x^+} f(x)^\top f(x^+) + \mathbb{E}_{x,x^-} \left(f(x)^\top f(x^-)\right)^2.
\end{equation}

Imposing a non-negative constraint on the outputs transforms the spectral contrastive loss into a non-negative form:
\begin{equation}\label{NCLloss}
    \small
    \mathcal{L}_{\mathrm{NCL}}(f) = -2 \mathbb{E}_{x, x^+} f_+(x)^\top f_+(x^+) + \mathbb{E}_{x, x^-} \left(f_+(x)^\top f_+(x^-)\right)^2,
\end{equation}
where $f_+(x) \geq 0$ for all $x \in X$.

Given that the normalized co-occurrence matrix $\bar{A}$ is a non-negative symmetric matrix, we apply SNMF to $\bar{A}$ to derive non-negative features $F_+$:
\begin{equation} \label{NMFloss}
    \mathcal{L}_{\mathrm{NMF}}(F) = \left\| \bar{A} - F_+ F_+^\top \right\|^2, \quad F_+ \geq 0.
\end{equation}

The SNMF problem in Eq.~\eqref{NMFloss} is equivalent to the non-negative spectral contrastive loss in Eq.~\eqref{NCLloss}, where $(F_+)_{x,:} = \sqrt{P(x)} f_+(x)^\top$. Expanding the terms:
\begin{equation}
\begin{split}
    \mathcal{L}_{\mathrm{NMF}}(F) 
    &= \sum_{x, x^+} \Bigg( 
        \frac{P(x, x^+)}{P(x)P(x^+)} \\
    &\qquad - \sqrt{P(x)}\, f_+(x)^\top 
        \sqrt{P(x^+)}\, f_+(x^+) 
    \Bigg)^2 \\[2pt]
    &= \sum_{x, x^+} 
        \frac{P(x, x^+)^2}{P(x)P(x^+)} \\ 
    &\quad - 2\, P(x, x^+)\, f_+(x)^\top f_+(x^+) \\ 
    &\quad + P(x)P(x^+)\, 
        \left( f_+(x)^\top f_+(x^+) \right)^2 \\[2pt]
    &= \mathcal{L}_{\mathrm{NCL}} + \text{const}.
\end{split}
\end{equation}

Thus, we have shown that contrastive learning with non-negative constraints performs symmetric non-negative matrix factorization on the normalized adjacency matrix.
\end{proof}

Integrating these findings, we underscore the equivalence between optimal K-means clustering and optimal SNMF and the equivalence between SNMF and NCL during the optimization. This provides a direct pathway for optimizing the underlying mechanisms of GCD models.

\section{Method}
In this section, we detail the proposed NN-GCD framework, which refines existing GCD methodologies.

\textbf{Task Definition:} The GCD task aims to accurately recognize both previously base (old) and novel categories \cite{vaze2022generalized}. The training dataset \(\mathcal{D}\) comprises a labeled subset \(\mathcal{D}_{l}\) and an unlabeled subset \(\mathcal{D}_{u}\). Specifically, the labeled set \(\mathcal{D}_l = \{(x_i, y_i)\}_{i=1}^{N_l} \subset \mathcal{X}_l \times \mathcal{Y}_l\) includes samples from base classes \(\mathcal{Y}_l = \mathcal{C}_{\text{old}}\), while the unlabeled set \(\mathcal{D}_u = \{x_{u_i}\}_{i=1}^{N_u} \subset \mathcal{X}_u\) includes samples from both base and novel classes \(\mathcal{Y}_u = \mathcal{C} = \mathcal{C}_{\text{old}} \cup \mathcal{C}_{\text{new}}\). Here, \(\mathcal{C}\) represents the complete set of categories within \(\mathcal{D}\), which can be predefined or identified through existing methodologies.

\begin{figure}[h] 
    \centering
    \includegraphics[width=0.4\textwidth]{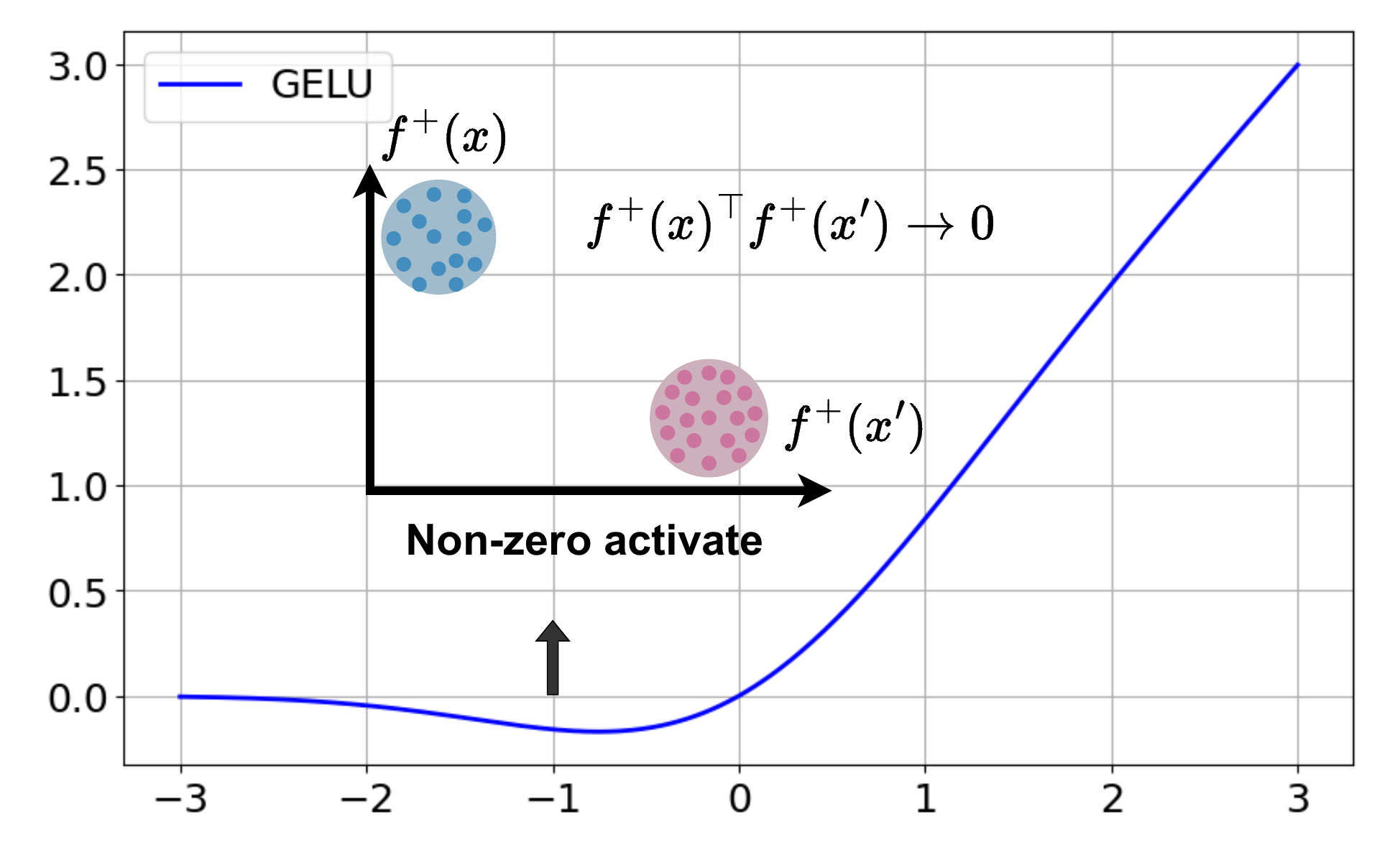}
    \caption{Non-zero activations formed with GELU.}
    \label{fig:main_GELU}
\end{figure}

\subsection{Non-negative Activated Neurons}
In the preceding theoretical section, we elucidate the equivalence between optimal clustering and optimal SNMF, as well as the optimal NCL. It proves that the optimal \(\bar{A}^*\) and clustering performance can be obtained through optimizing NCL. Therefore, in this subsection, we propose a more optimal network architecture based on NCL tailored for the GCD scenario, ensuring the non-negativity constraint.

To enforce this constraint, the neural network outputs are reparameterized. Following a standard neural network encoder \( f \)—comprising a feature extractor \( F \) and a multi-layer perceptron (MLP) projection head \( g \)—as used in contrastive learning, a non-negative transformation \(\phi(\cdot)\) (satisfying \(\phi(x) \ge 0, \forall x\)) is introduced:
\begin{equation}
    f_{+}(x) = \phi(f(x)).
\end{equation}

Potential choices of $\phi$ include ReLU, sigmoid, etc. We find that the simple ReLU function, $\text{ReLU}(x) = \max(x, 0)$, compared to other operators with non-negative outputs, imposes a very stringent constraint through its zero-output effect for inputs less than zero. 

Although ReLU-activated features exhibit absolute non-negativity, they often encounter the issue of dead neurons during training \cite{lu2019dying}. During NCL optimization, the similarity between samples of different classes is driven to zero: \( f^+(x)^\top f^+(x') \rightarrow 0 \). However, the problem of dead neurons reduces the feature representation capability, forming all-zero activations that lead to all-zero features. This results in \(\bar{A}\) being sufficiently sparse but with weak class separability, significantly hindering the generalization required for GCD tasks and causing confusion among different classes.
To solve this problem, we propose to use GELU\cite{hendrycks2016gaussian}, which can be considered a smooth form of ReLU:
\begin{equation} 
\text{GELU}(x) = x \cdot \Phi(x) = x\cdot\frac{1}{2} \left( 1 + \operatorname{erf}\left( \frac{x}{\sqrt{2}} \right)\right),
\end{equation}
where \(\Phi(x)\) is the cumulative distribution function of the standard normal distribution.

According to \cite{lee2023mathematical}, unlike ReLU, GELU exhibits Lipschitz continuity and smoothness, which mitigates the dead neuron issue and benefits model generalization \cite{virmaux2018lipschitz}. As shown in Fig.~\ref{fig:main_GELU}, GELU preserves the input values, enabling non-zero activations during optimization. This allows \( f^+(x) \) and \( f^+(x') \) to form non-zero orthogonal solutions. The non-zero orthogonality of features from different classes ensures that inter-class cosine similarity approaches zero while intra-class cosine similarity remains close to one, thereby preventing feature collapse, especially between base and novel classes. Ultimately, this achieves class separability of \(\bar{A}^*\), significantly enhancing GCD performance.

To more reliably generate non-zero activations, we stop the gradient updates for the teacher head \(g_t\) and update the teacher head parameters \(\Theta_t = \{\theta_t, m_t\}\) \cite{grill2020bootstrap} using the Exponential Moving Average (EMA) method, specifically updating \(\Theta_t\) as \(\Theta_t = \omega(t)\Theta_t + (1 - \omega(t))\Theta\), where \(\Theta\) refers to the trainable parameters of the student head \(g_s\). The hyperparameter \(\omega(t)\) is updated using a cosine decay schedule, defined as \(\omega(t) = \omega_{\text{max}} - (1 - \omega_{\text{min}})\cos\left(\frac{\pi t}{T + 1}\right)/2\).

\subsection{NMF NCE Loss}
The introduction of the Non-negative Activated Neurons mechanism enforces the non-negative output constraint in NCL, enabling the use of contrastive learning-based loss functions and gradient descent to find the optimal \(\bar{A}^*\). To learn the representations of the data, we follow SimGCD \cite{wen2023parametric} and combine self-supervised and supervised learning. We combine the SimCLR \cite{chen2020simple} loss and the cross-entropy loss between predictions and pseudo labels:
\begin{equation}
\mathcal{L}_{\mathit{SSL}} = \mathcal{L}_{\mathit{SimCLR}} + \mathcal{L}_{\mathit{pseudo}}.
\end{equation}

For labeled data, we use a supervised learning loss that combines the SupCon \cite{khosla2020supervised} loss and the cross-entropy loss between predicted and ground-truth labels:
\begin{equation}
\mathcal{L}_{\mathit{SL}} = \mathcal{L}_{\mathit{SupCon}} + \mathcal{L}_{\mathit{CE}}.
\end{equation}

We calculate the SimCLR loss as follows:
\begin{equation}
\mathcal{L}_{\mathit{SimCLR}} = -\frac{1}{|B|} \sum_{i \in B} \log \frac{\exp(h_i^\top h_i' / \tau)}{\sum_{i \neq n} \exp(h_i^\top h_n' / \tau)},
\end{equation}
where \(x_i\) and \(x_i'\) are two views of the same image in a mini-batch \(B\), \(h_i = g(F(x_i))\), and \(\tau\) is a temperature value. \(F\) is the feature extractor ViT, and \(g\) is the MLP projection head.

The SupCon loss is calculated as follows:
\begin{equation}
\small
\mathcal{L}_{\mathit{SupCon}} = -\frac{1}{|B_L|} \sum_{i \in B_L} \frac{1}{|N_i|} \sum_{q \in N_i} \log \frac{\exp(h_i^\top h_q' / \tau)}{\sum_{i \neq n} \exp(h_i^\top h_n' / \tau)},
\end{equation}
where \(B_L\) is the labeled subset of \(B\), and \(N_i\) is the set of indices of images sharing the same label as \(x_i\).

\begin{figure}[h!] 
    \centering
    \begin{subfigure}[b]{0.23\textwidth}
        \centering
        \includegraphics[width=\textwidth]{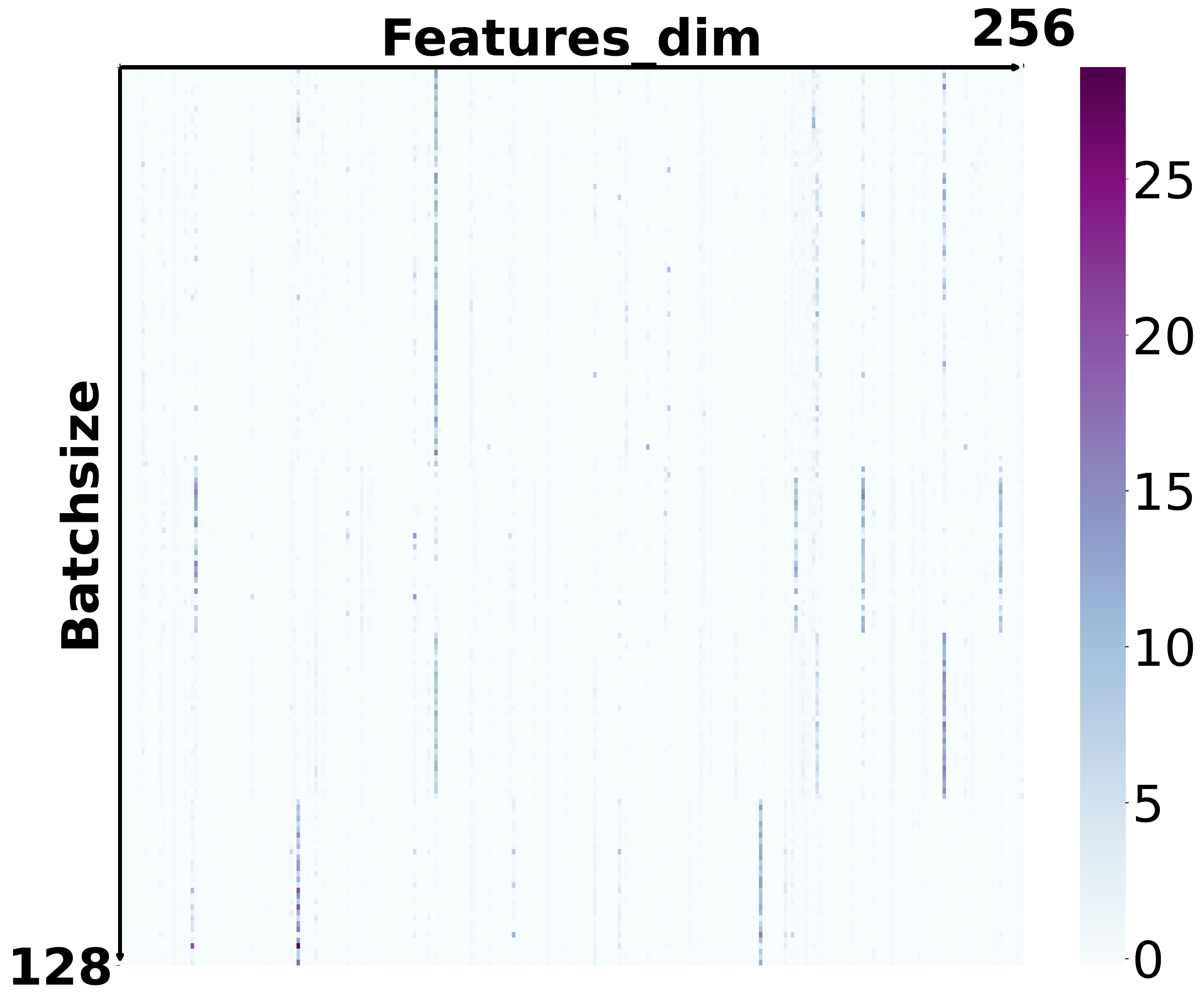}
    \end{subfigure}
    \hspace{0.0001\textwidth}
    \begin{subfigure}[b]{0.23\textwidth}
        \centering
        \includegraphics[width=\textwidth]{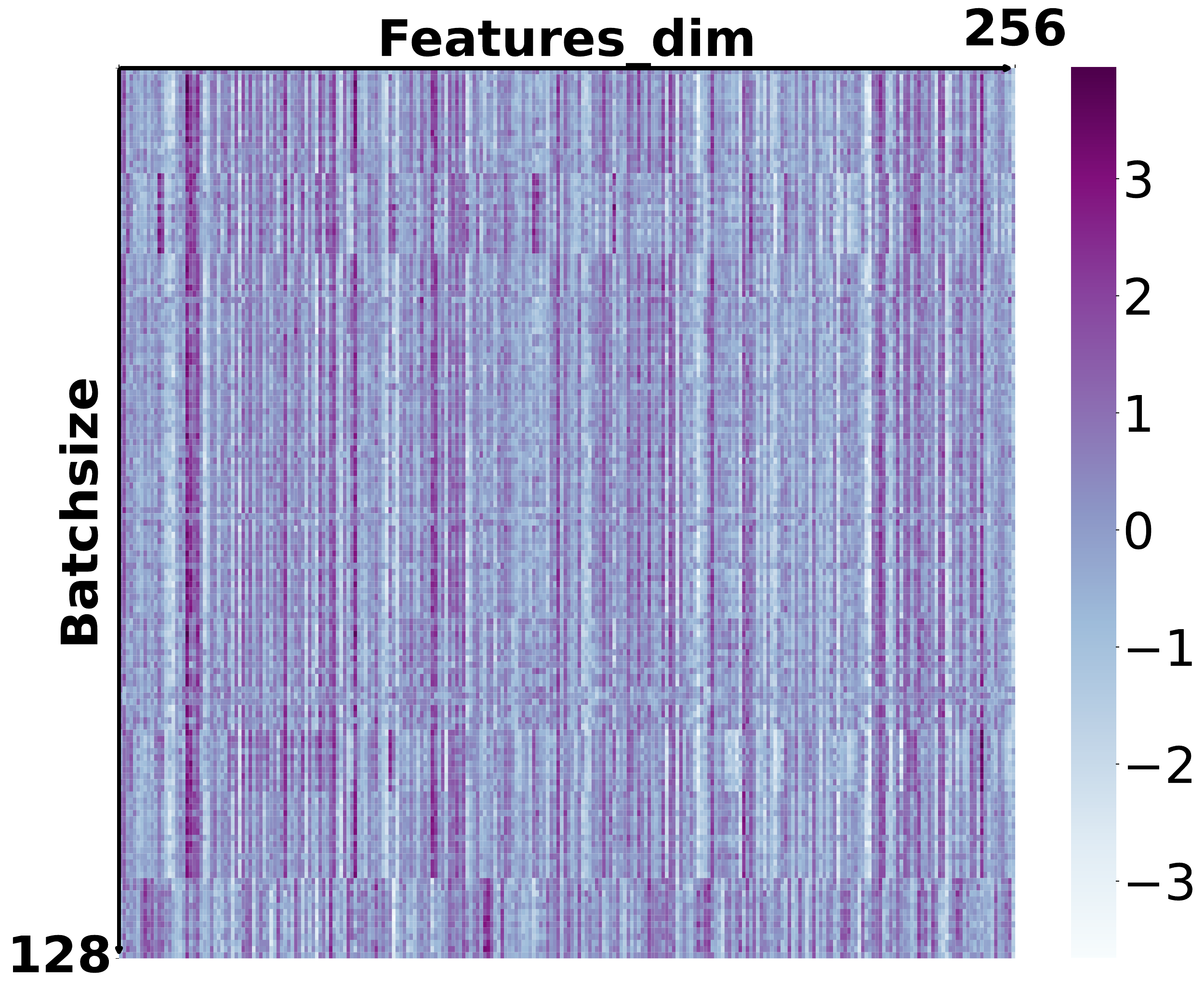}
    \end{subfigure}
    \hspace{0.0001\textwidth}
    \caption{An analytical visualization contrasting the features extracted via our proposed methodology with those obtained through SimGCD \cite{wen2023parametric}.}
    \label{fig:main}
\end{figure}

\begin{figure}[h!] 
    \centering
    \begin{subfigure}[b]{0.23\textwidth}
        \centering
        \includegraphics[width=\textwidth]{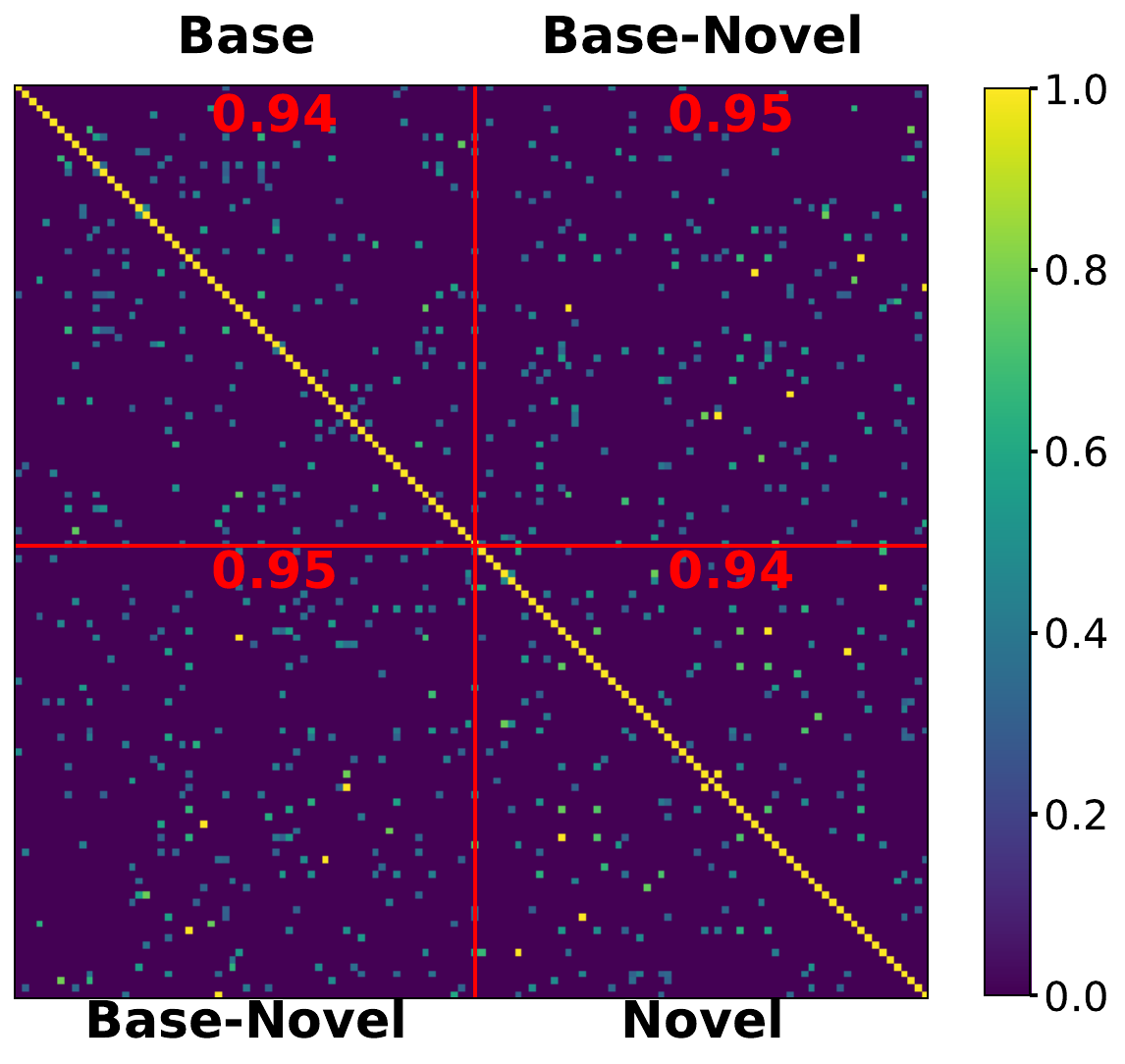}
    \end{subfigure}
    \hspace{0.0001\textwidth}
    \begin{subfigure}[b]{0.23\textwidth}
        \centering
        \includegraphics[width=\textwidth]{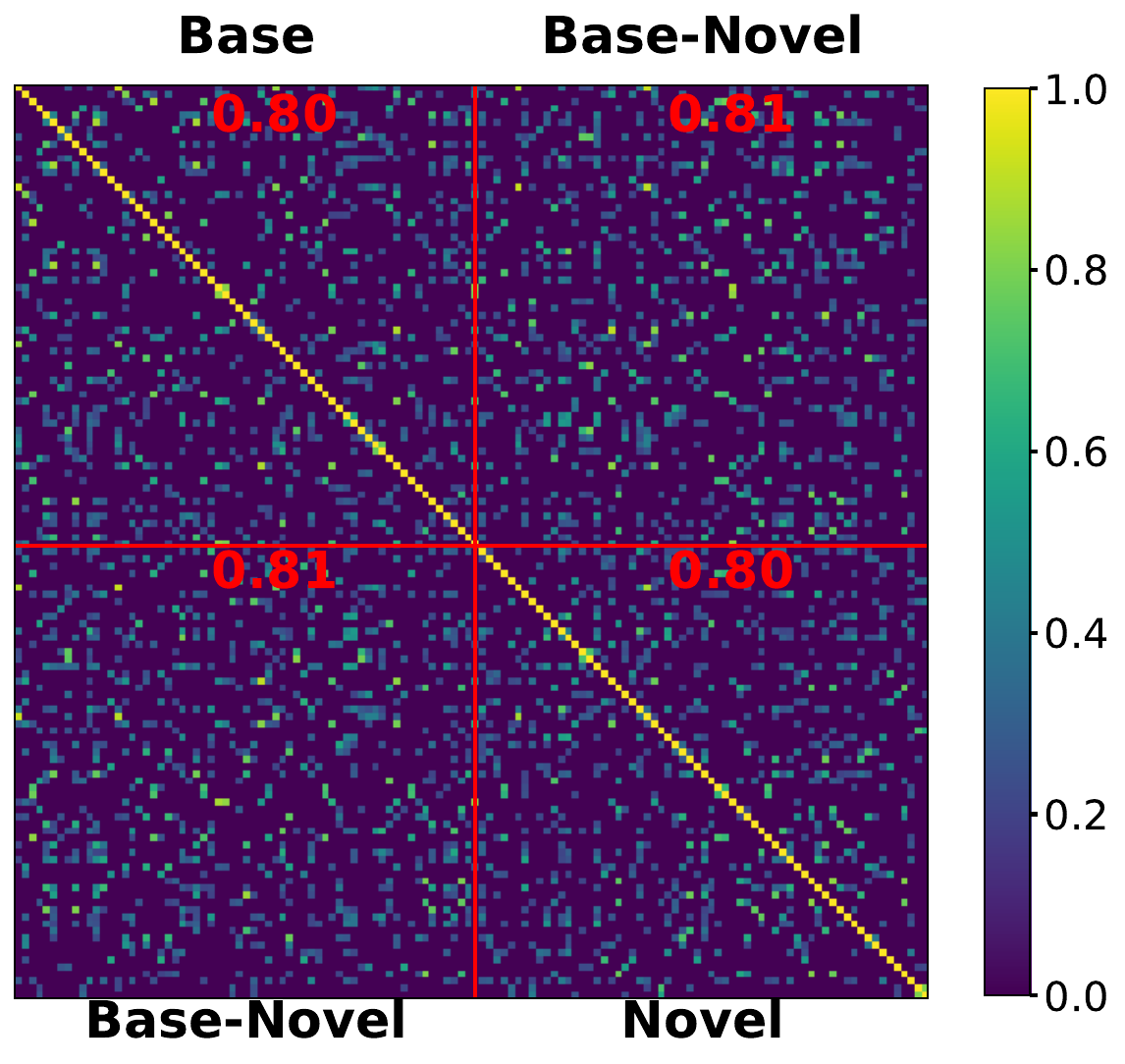}
    \end{subfigure}
    \caption{A comparative visualization of the \(\bar{A}\) matrix computed from features extracted by our NN-GCD versus those derived using SimGCD \cite{wen2023parametric}.}
    \label{fig:main_matrix}
\end{figure}

We utilize two additional loss functions to train the feature extractor, following \cite{wen2023parametric} and applied in a self-distillation fashion \cite{assran2022masked}. Specifically, for a total number of categories \( K \), we randomly initialize a set of prototypes \( C = \{c_1, \dots, c_K\} \), each representing one category. During training, we calculate the soft label for each augmented view \( x_i \) by performing softmax on the cosine similarity between the hidden feature \( z_i = F(x_i) \) and the prototypes \( C \), scaled by \( 1/\tau_s \):
\begin{equation}
p_i^{(k)} = \frac{ \exp\left( (1/\tau_s) \left( \left( z_i / \| z_i \|_2 \right)^\top \left( c_k / \| c_k \|_2 \right) \right) \right) }{ \displaystyle\sum_{k'} \exp\left( (1/\tau_s) \left( \left( z_i / \| z_i \|_2 \right)^\top \left( c_{k'} / \| c_{k'} \|_2 \right) \right) \right)}.
\end{equation}

The soft pseudo-label \( q'_i \) is generated from another view \( x_i \) using a sharper temperature \( \tau_t \). The loss function for predictions and pseudo-labels is the cross-entropy loss:
\begin{equation}
\mathcal{L}_{\mathit{pseudo}} = \frac{1}{|B|} \sum_{i \in B} \ell(q'_i, p_i) - \epsilon H(p).
\end{equation}

For known labels, the loss is defined as:
\begin{equation}
\mathcal{L}_{\mathit{CE}} = \frac{1}{|B_L|} \sum_{i \in B_L} \ell(y_i, p_i),
\end{equation}
where \( y_i \) denotes the one-hot label of \( x_i \). 

For the unsupervised objective, we also adopt a mean-entropy maximization regularizer \cite{assran2022masked}. Here, \( p = \frac{1}{2|B|} \sum_{i \in B} (p_i + p'_i) \) denotes the mean prediction of a batch, and the entropy is defined as:
\begin{equation}
H(p) = -\sum_k p(k) \log p(k).
\end{equation}

Although most existing methods \cite{vaze2022generalized, wen2023parametric} utilize the InfoNCE loss in $L_{\mathit{SimCLR}}$, it has certain limitations.

\textbf{Overemphasis on extreme samples}: The InfoNCE loss often assigns disproportionately large weights to the hardest negative samples with the highest similarity, causing the model to overfit to these extreme samples during training and potentially overlook other equally informative samples.

\textbf{Sensitivity to outliers}: The InfoNCE loss is highly sensitive to outliers, as these outliers can be assigned excessively high weights when they appear in regions of high similarity, causing significant fluctuations in training performance.

Based on these limitations, we develop \textbf{NMF NCE}:
\small
\begin{equation}
    \omega(f(x), \mu, \sigma) \propto \frac{1}{\sigma \sqrt{2\pi}} \exp \left[ -\frac{1}{2} \left( \frac{f(x) - \mu}{\sigma} \right)^2 \right],
\end{equation}
\begin{equation}
    \begin{aligned}
        \mathcal{L}_{\text{NNCE}} = &- \mathbb{E}_{x, x^{+}} \left[ \frac{f(x^+)}{\tau} \right]\\ 
        &+ \log \mathbb{E}_{x, x^{-}} \left[ \omega(f(x), \mu, \sigma) \frac{e^{f(x)/\tau}}{\mathbb{E} \left[ \omega(f(x), \mu, \sigma) \right]} \right],
    \end{aligned} 
\end{equation}
\normalsize
where \(\mu\) and \(\sigma\) are controllable hyperparameters. \(\mu\) defines the central region of the weight distribution, assigning greater weights to samples nearer to it. \(\sigma\) controls the distribution's peak height, with a smaller \(\sigma\) resulting in more pronounced weight differences between samples.

The NMF NCE loss strategically assigns weights to negative instances, enabling the model to prioritize primary samples over outliers. This methodological approach mitigates confusion between the base and novel domains, enhances class separability, and achieves base-level insulation.

\subsection{Hybrid Sparse Regularization}
The optimization target \(\bar{A}^*\) in NCL is highly sparse, necessitating the minimization of similarity among dissimilar negative samples (\(f^+(x)^\top f^+(x') \rightarrow 0\)). This results in \(f^+(x)\) and \(f^+(x')\) forming non-zero orthogonal solutions, leading to zero activation in the irrelevant dimensions of the features, thus requiring inherent sparsity in NCL features. Imposing sparsity constraint regularization on NCL maintains feature sparsity during optimization, enabling better convergence to the optimal solution. To achieve this, we propose the Hybrid Sparse Regularization:
\begin{equation}
\mathcal{L}_{HSR}(W) =\gamma(\beta \|W\|_1 + (1-\beta) (\|W\|_{2,1} - \|W\|_F^2)),
\end{equation}
Here, \(\beta\) and \(\gamma\) act as balancing parameters. The regularization term mitigates the deficiencies of the \(\mathbf{l}\)-norm and handles feature redundancy, a common issue with the \(\mathbf{l}_{2,1}\)-norm.

The purpose of $\|W\|_1$ is to enhance the sparsity of the parameter matrix, while the subsequent term$(\|W\|_{2,1} - \|W\|_F^2)$, maintaining row sparsity while controlling the overall energy of the matrix. The modified \( l_{2,1} \) norm, \( \|W\|_{2,1} \), aims to constrain row sparsity, thereby reducing redundancy, while the squared Frobenius norm, \( \|W\|_F^2 \), controls the overall energy of the parameter matrix.

The total loss function for NN-GCD is then equal to:
\begin{equation}
    \mathcal{L}_{\mathit{NN-GCD}} = (1-\lambda)\mathcal{L}_{\mathit{SSL}} + \lambda \mathcal{L}_{\mathit{SL}} + \mathcal{L}_{HSR}(W).
\end{equation}

By employing the aforementioned methods, the optimal features and \(\bar{A}^*\) are achieved. The results are visualized as follows: Figure~\ref{fig:main} highlights our method's non-negative, sparse features, marked by mostly zero eigenvalues and selective large positives. By refining NCL through a specialized loss function and sparse regularization, we achieve optimal \(\bar{A}^*\) with class-separable features. Additionally, Figure~\ref{fig:main_matrix} underscores superior sparsity in both base and novel regions, nearing unity in the base-novel area, thus outperforming predecessors.

\begin{table}[h!]
\centering
\begin{tabular}{cccccc}
\toprule
\textbf{Dataset} & \textbf{Balance} & \multicolumn{2}{c}{\textbf{Labelled}} & \multicolumn{2}{c}{\textbf{Unlabelled}} \\ 
\cmidrule(lr){3-4} \cmidrule(lr){5-6}
& & \#Image & \#Class & \#Image & \#Class \\ 
\midrule
CIFAR10 & \checkmark & 12.5K & 5 & 37.5K & 10 \\ 
CIFAR100 & \checkmark & 20.0K & 80 & 30.0K & 100 \\ 
ImageNet-100 & \checkmark & 31.9K & 50 & 95.3K & 100 \\ 
CUB & \checkmark & 1.5K & 100 & 4.5K & 200 \\ 
Stanford Cars & \checkmark & 2.0K & 98 & 6.1K & 196 \\ 
FGVC-Aircraft & \checkmark & 1.7K & 50 & 5.0K & 100 \\ 
Herbarium 19 & $\times$ & 8.9K & 341 & 25.4K & 683 \\ 
\bottomrule
\end{tabular}
\caption{Statistics of the datasets we evaluate on.}
\label{tab: datasets introduce}
\end{table}

\begin{table*}[h!]
\centering
\caption{Category discovery accuracy (ACC) on the Semantic Shift Benchmark. The reported results include prior work using DINO initialization, alongside reimplemented GCD baselines, SimGCD, and $\mu$GCD with DINOv2 pre-training (indicated with *). The highest-performing results are highlighted in bold, and the second-best results are underlined.}
\fontsize{8.5}{11}\selectfont 
\begin{tabular}{l>{\columncolor{Ocean}}ccc>{\columncolor{Ocean}}ccc>{\columncolor{Ocean}}ccc>{\columncolor{Ocean}}cc}
\toprule
\multicolumn{1}{c}{\textbf{method}} & \multicolumn{3}{c}{\textbf{CUB}} & \multicolumn{3}{c}{\textbf{Stanford Cars}} & \multicolumn{3}{c}{\textbf{Aircraft}} & \multicolumn{1}{c}{\textbf{Avg}}\\
\cmidrule(lr){2-4} \cmidrule(lr){5-7} \cmidrule(lr){8-10}
& \textbf{All} & \textbf{Old} & \textbf{New} & \textbf{All} & \textbf{Old} & \textbf{New} & \textbf{All} & \textbf{Old} & \textbf{New} & \textbf{All} \\
\midrule
GCD \cite{vaze2022generalized} & 31.3 & 56.6 & 48.7 & 39.0 & 57.6 & 29.9 & 45.0 & 41.1 & 46.9 & 45.1\\
XCon \cite{fei2022xcon} & 52.1 & 54.3 & 51.0 & 40.5 & 58.8 & 31.7 & 47.7 & 44.4 & 49.4 & 46.8\\
DCCL \cite{pu2023dynamic} & 63.5 & 60.8 & 64.9 & 43.1 & 55.7 & 36.2 & - & - & - & - \\
PIM \cite{chiaroni2023parametric} & 62.7 & 75.7 & 56.2 & 43.1 & 66.9 & 31.6 & 42.3 & 56.1 & 34.8 & 49.4\\
CiPR \cite{hao2023cipr} & 57.1 & 58.7 & 55.6 & 47.0 & 61.5 & 40.1 & 36.8 & 45.4 & 32.6 & 50.0 \\
PromptCAL \cite{zhang2023promptcal} & 62.9 & 64.4 & 62.1 & 50.2 & 70.1 & 40.6 & 52.2 & 52.2 & 52.3 & 55.1\\
InfoSieve \cite{rastegar2024learn} & \underline{69.4} & \textbf{77.9} & \underline{65.2} & 56.3 & 63.7 & \underline{52.5} & 55.7 & \textbf{74.8} & 46.4 & 60.5\\ 
SimGCD \cite{wen2023parametric} & 60.3 & 65.6 & 57.7 & 53.8 & 71.9 & 45.0 & 54.2 & 59.1 & 51.8 & 56.1\\
$\mu$GCD \cite{vaze2024no} & 65.7 & 68.0 & 64.6 & 56.5 & 68.1 & 50.9 & 53.8 & 55.4 & 53.0 & 58.7\\
SPTnet \cite{wang2024sptnet} & 65.8 & 68.8 & 65.1 & \underline{59.0} & \textbf{79.2} & 49.3 & \underline{59.3} & 61.8 & \underline{58.1} & \underline{61.4} \\
\textbf{NN-GCD} (Ours) & \textbf{72.4} & \underline{76.3} & \textbf{70.5} & \textbf{65.4} & \underline{77.5} & \textbf{59.5} & \textbf{60.5} & \underline{65.3} & \textbf{58.2} & \textbf{66.1}\\
\midrule
k-means* & 67.6 & 60.6 & 71.1 & 29.4 & 24.5 & 31.8 & 18.9 & 16.9 & 19.9 & 38.6\\
GCD* \cite{vaze2022generalized} & 71.9 & 71.2 & 72.3 & 67.5 & 67.8 & 64.7 & 55.4 & 47.9 & 59.2 & 64.3\\

SimGCD* \cite{wen2023parametric} & 71.5 & \underline{78.1} & 68.3 & 71.3 & 81.9 & 66.6 & 63.9 & \underline{69.9} & 60.9 & 69.0\\
$\mu$GCD* \cite{vaze2024no} & \underline{74.0} & 75.9 & \underline{73.1} & \underline{76.1} & \textbf{91.0} & \underline{68.9} & \underline{66.3} & 68.7 & \underline{65.1} & \underline{72.1}\\
\textbf{NN-GCD}* (Ours) & \textbf{83.1} & \textbf{82.5} & \textbf{83.4} & \textbf{80.4} & \underline{88.9} & \textbf{76.3} & \textbf{75.6} & \textbf{78.2} & \textbf{74.3} & \textbf{79.7}  \\
\bottomrule
\end{tabular}
\label{tab:semantic_shift_benchmark}
\end{table*}

\begin{table*}[h!]
\centering
\caption{Category discovery accuracy (ACC) on the generic image recognition datasets. Results are derived from prior work utilizing DINO initialization. The highest-performing results are highlighted in bold, and the second-best results are underlined.}
\fontsize{8.5}{11}\selectfont 
\begin{tabular}{l>{\columncolor{Ocean}}ccc>{\columncolor{Ocean}}ccc>{\columncolor{Ocean}}ccc>{\columncolor{Ocean}}c}
\toprule
\multicolumn{1}{c}{\textbf{Method}} & \multicolumn{3}{c}{\textbf{CIFAR10}} & \multicolumn{3}{c}{\textbf{CIFAR100}} & \multicolumn{3}{c}{\textbf{ImageNet-100}} & \multicolumn{1}{c}{\textbf{Average}}\\
\cmidrule(lr){2-4} \cmidrule(lr){5-7} \cmidrule(lr){8-10}
\multicolumn{1}{c}{} &\textbf{All} & \textbf{Old} & \textbf{New} & \textbf{All} & \textbf{Old} & \textbf{New} & \textbf{All} & \textbf{Old} & \textbf{New} & \textbf{All} \\
\midrule
GCD \cite{vaze2022generalized} & 91.5 & \textbf{97.9} & 88.2 & 70.8 & 77.6 & 57.0 & 74.1 & 89.8 & 66.3 & 78.8 \\
DCCL \cite{pu2023dynamic} & 96.3 & 96.5 & 96.9 & 75.3 & 76.8 & 70.2 & 80.5 & 90.5 & 76.2 & 84.0 \\
XCon \cite{fei2022xcon} & 96.0 & 97.3 & 95.4 & 74.2 & 81.2 & 60.3 & 77.6 & 93.5 & 69.7 & 82.6\\
PIM \cite{chiaroni2023parametric} & 94.7 & 97.4 & 93.3 & 78.3 & \underline{84.2} & 66.5 & 83.1 & \textbf{95.3} & 77.0 & 85.3 \\
CiPR \cite{hao2023cipr} & \underline{97.7} & 97.5 & 97.7 & \textbf{81.5} & \underline{82.4} & \textbf{79.7} & 80.5 & 84.9 & 78.3 & 86.6 \\
PromptCAL \cite{zhang2023promptcal} & \textbf{97.9} & 96.6 & \underline{98.5} & 81.2 & 84.2 & 75.3 & 83.1 & 92.7 & 78.3 & 87.7\\
InfoSieve \cite{rastegar2024learn} & 94.8 & \underline{97.7} & 93.4 & 78.3 & 82.2 & 70.5 & 80.5 & \underline{93.8} & 73.8 & 84.5 \\

SimGCD \cite{wen2023parametric} & 97.1 & 95.1 & 98.1 & 80.1 & 81.2 & 77.8 & 83.0 & 93.1 & 77.9 & 86.7\\

SPTnet \cite{wang2024sptnet} & 97.3 & 95.0 & \textbf{98.6} & \underline{81.3} & \textbf{84.3} & 75.6 & \textbf{85.4} & 93.2 & \underline{81.4} & \textbf{88.0} \\ 
\midrule
\textbf{NN-GCD} (Ours) &  97.3 & 95.0 & \textbf{98.6} & 81.2 & 82.4 & \underline{78.8} & \underline{85.2} & 91.5 & \textbf{82.0} & \underline{87.9} \\
\bottomrule
\end{tabular}
\label{tab:generic_image_recognition_datasets}
\end{table*}

\begin{table}[h!]
\centering
\caption{Category discovery accuracy (ACC) on the more challenging datasets. Results are derived from prior work utilizing DINO initialization. The best results are highlighted in bold, and the second-best results are underlined.}
\fontsize{10}{11}\selectfont 
\begin{tabular}{l>{\columncolor{Ocean}}cccc}
\toprule
\multicolumn{1}{c}{\textbf{Method}} & \multicolumn{3}{c}{\textbf{Herbarium 19}} \\
\cmidrule(lr){2-4} 
\multicolumn{1}{c}{} & \textbf{All} & \textbf{Old} & \textbf{New}  \\
\midrule
GCD \cite{vaze2022generalized} & 35.4 & 51.0 & 27.0 \\
CiPR \cite{hao2023cipr} & 36.8 & 45.4 & 32.6 \\
PIM \cite{chiaroni2023parametric} & 42.3 & 56.1 & 34.8  \\
PromptCAL \cite{zhang2023promptcal} & 37.0 & 52.0 & 28.9 \\
InfoSieve \cite{rastegar2024learn} & 41.0 & 55.4 & 33.2 \\ 
SimGCD \cite{wen2023parametric} & \textbf{44.0} & \underline{58.0} & \textbf{36.4} \\
SPTnet \cite{wang2024sptnet} & 43.4 & \textbf{58.7} & 35.2 \\
\midrule
\textbf{NN-GCD} &\underline{43.5} & \textbf{58.7} & \underline{36.0} \\
\bottomrule
\end{tabular}
\label{tab:herbarium19}
\end{table}

\section{Experiments}
\textbf{Datasets.} In this section, we validate the effectiveness of our method through extensive testing on seven diverse image recognition benchmarks. These benchmarks include the Semantic Shift Benchmark (SSB), featuring fine-grained datasets: CUB \cite{welinder2010caltech}, Stanford Cars \cite{krause20133d}, and Aircraft \cite{maji2013fine}, generic object recognition datasets: CIFAR-10/100 \cite{krizhevsky2009learning}, and ImageNet-100 \cite{tian2020contrastive}, as well as the more challenging dataset Herbarium 19\cite{tan2019herbarium}. The SSB introduces complexities with fine-grained categories, challenging method robustness in realistic settings. For a fair comparison, we follow the prevalent GCD setting \cite{vaze2022generalized}, dividing each dataset into labeled and unlabeled subsets. Specifically, 80\% of CIFAR-100 and 50\% of other datasets' base (old) samples form \(\mathcal{D}_{l}\), with the remainder constituting \(\mathcal{D}_{u}\). As shown in Tab.~\ref{tab: datasets introduce}, both the SSB and generic object recognition datasets are balanced datasets. However, the number of samples per class in the SSB is relatively smaller compared to the generic object recognition datasets. Herbarium 19 introduces the challenge of data imbalance, where there are many categories with few samples, while others have abundant samples, resulting in a long-tail distribution that makes it difficult for the model to learn features of those underrepresented categories. 

\textbf{Evaluation protocol.} We evaluate the model performance with clustering accuracy (ACC) following standard practice \cite{vaze2022generalized}. During evaluation, given the ground truth $y^*$ and the predicted labels $\hat{y}$, the ACC is calculated as:
\begin{equation}
    \text{ACC} = \frac{1}{M} \sum_{i=1}^{M} \mathbf{1}(y^*_i = p(\hat{y}_i)),
\end{equation}
where $M = |\mathcal{D}_u|$, and $p$ is the optimal permutation that matches the predicted cluster assignments to the ground truth class labels.

\textbf{Implementation details.} Following \cite{vaze2022generalized, wen2023parametric}, we adopt ViT-B/16 as the backbone network and use the [CLS] token with a dimension of 768 from the encoder output as the feature representation. The network is either pre-trained using DINOv1 \cite{caron2021emerging} on unlabelled ImageNet 1K \cite{krizhevsky2012imagenet}, or pre-trained using DINOv2 \cite{oquab2023dinov2} on unlabelled ImageNet 22K. We use a batch size of 128 for training, and the training lasts for 200 epochs. The learning rate is set to 0.1 and decayed using the cosine schedule \cite{loshchilov2016sgdr}. The EMA update schedule follows \cite{vaze2024no}, where $\omega_{\text{min}}$ is set to 0.7 and $\omega_{\text{max}}$ is set to 0.99. In our representation learning, the balance factor is set to 0.35, $\tau_s$ and $\tau_t$ are set to 0.1 and 0.007, for the pseudo-labels cross-entropy loss, $\epsilon$ is set to 1, for the NCE Loss, $\tau$ is set to 0.5, and $\mu$ and $\sigma$ are set to 0.1 and 1.0, respectively. The regularization balancing factors $\gamma$ and $\beta$ are set to \(3e^{-5}\) and 0.6, respectively. We use PyTorch to implement our method, and the experiments are conducted on a single RTX 4090 GPU.

\subsection{Comparison with the State-of-the-arts}
We compared the proposed method with the current state-of-the-art methods, which are categorized into non-parametric and parametric approaches. Non-parametric methods involve using semi-supervised k-means or other clustering techniques in the final classification stage, with GCD \cite{vaze2022generalized} being a representative example. In contrast, parametric methods, introduced by SimGCD \cite{wen2023parametric}, employ a parametric classification head. The non-parametric methods we compared include GCD \cite{vaze2022generalized}, XCon \cite{fei2022xcon}, DCCL \cite{pu2023dynamic}, PIM \cite{chiaroni2023parametric}, CiPR \cite{hao2023cipr}, PromptCAL \cite{zhang2023promptcal}, and InfoSieve \cite{rastegar2024learn}, while the parametric methods include SimGCD \cite{wen2023parametric}, $\mu$GCD \cite{vaze2024no}, and SPTnet \cite{wang2024sptnet}. Our NN-GCD, which builds on SimGCD, falls under the parametric category.

\textbf{Evaluation of Semantic Shift Benchmark (SSB).} Our method, utilizing both DINO and DINOv2 initialization, demonstrates significant advantages over state-of-the-art approaches on the SSB benchmark datasets. As shown in Tab.~\ref{tab:semantic_shift_benchmark}, applying DINO pretrained weights results in an average improvement of 10\% across all SSB datasets, with a 4.7\% increase over SPTnet. This highlights the superior performance of our method across various benchmark datasets, particularly in fine-grained classification tasks.

On the CUB dataset, NN-GCD achieves notably strong results, with an accuracy of 72.4\% in the "All" category, surpassing InfoSeive by 3.0\%. In the "New" category, the accuracy reaches 70.5\%, outperforming InfoSeive by 5.3\%. Although NN-GCD ranks second in the "Old" category with 76.3\%, its stable performance suggests robust handling of both new and old categories. For the Stanford Cars dataset, NN-GCD achieves 65.4\% accuracy in the "All" category, surpassing SPTnet by 6.4\%, and 59.5\% in the "New" category, marking a 7.0\% improvement over InfoSeive. On the Aircraft dataset, NN-GCD achieves 60.5\% accuracy in the "All" category and 58.2\% in the "New" category, slightly outperforming SPTnet. Overall, NN-GCD sets a new state-of-the-art on the CUB, Stanford Cars, and Aircraft datasets, excelling in both the "All" and "New" categories while maintaining strong performance in the "Old" category. The NMF NCE loss effectively balances performance across both new and old categories, and the integration of Non-negative Activated Neurons and Hybrid Sparse Regularization further optimizes \(\bar{A}^*\), enhancing fine-grained classification performance.

Additionally, further experiments on the SSB benchmark using DINOv2 initialization show NN-GCD achieving state-of-the-art results, surpassing the baseline by 10.7\% and outperforming \(\mu\)GCD by 7.6\%. In the CUB, Aircraft, and Stanford Cars datasets, NN-GCD demonstrated substantial improvements, underscoring the critical role of pre-trained models like DINOv2 in enhancing the performance of our method, particularly in handling new categories, further validating the effectiveness of our approach.

\textbf{Evaluation of generic image recognition datasets.} The results on CIFAR-10, CIFAR-100, and ImageNet-100, as shown in Tab.~\ref{tab:generic_image_recognition_datasets}, demonstrate the effectiveness of NN-GCD in comparison with state-of-the-art methods. On CIFAR-10, NN-GCD achieved accuracies of 97.3\% in the ``All'' category and 98.6\% in the ``New'' category, performing on par with SPTnet. This shows that our method is effective in extracting features and delivering high classification performance, even for low-resolution images.

For CIFAR-100, which poses more complexity due to its larger number of categories, NN-GCD achieved 78.8\% in the ``New'' category, outperforming most methods. Although slightly behind CiPR in the ``Old'' category, the method exhibits strong generalization and classification ability. The low resolution of CIFAR-10 and CIFAR-100 (32×32) may limit further performance gains, particularly in sparsification, but NN-GCD remains competitive. On ImageNet-100, NN-GCD reached 82.0\% accuracy in the ``New'' category and 85.2\% in the ``All'' category, demonstrating adaptability even on large-scale datasets with a highly optimized DINO backbone.

In summary, NN-GCD shows strong generalization and robustness across all three datasets, excelling in both known and unseen category classification. Its strengths in sparsification and feature extraction make it highly competitive across a range of image recognition tasks.

\textbf{Evaluation of Herbarium 19 dataset.} The results on the Herbarium 19 dataset, as shown in Tab.~\ref{tab:herbarium19}, demonstrate that NN-GCD delivers competitive performance across categories. In the ``All'' category, NN-GCD achieved 43.5\% accuracy, slightly surpassing SPTnet (43.4\%) and closely matching SimGCD (44.0\%), showing strong overall classification ability. In the ``Old'' category, NN-GCD and SPTnet both achieved the highest accuracy of 58.7\%, outperforming SimGCD, highlighting NN-GCD's effectiveness in leveraging known categories.

For the ``New'' category, NN-GCD achieved 36.0\%, slightly behind SimGCD (36.4\%) but still surpassing SPTnet and other methods, indicating strong generalization to unseen categories. The imbalanced nature of the dataset, with large variations in sample sizes, likely limits the model's learning of new categories, contributing to slightly lower performance. Nevertheless, NN-GCD remains competitive with state-of-the-art methods, excelling in the ``Old'' category and demonstrating robust generalization in the ``New'' category.

\subsection{Ablation Study}
In this ablation study, we evaluated the contributions of three key components: Non-negative Activated Neurons, NMF NCE loss ($\mathcal{L}_{\text{NNCE}}$), and Hybrid Sparse Regularization ($\mathcal{L}_{HSR}$), analyzing both their individual and combined effects on performance. As shown in Tab.~\ref{tab:methods_comparison}, the results indicate that each component contributes positively to the model's accuracy.

First, Non-negative Activated Neurons serve as a foundational module, improving feature representation by activating positively correlated features while suppressing negative signals. This resulted in accuracy gains of 4.3\%, 3.5\%, and 4.6\% in the "All," "Old," and "New" categories, respectively. These results demonstrate the module's effectiveness in distinguishing between old and new categories, particularly in improving generalization in the latter.

Next, $\mathcal{L}_{HSR}$ enhances feature differentiation by sparsifying representations, resulting in a 6.3\% improvement in the "All" category, with 6.7\% and 6.1\% improvements in the "Old" and "New" categories, respectively. This highlights $\mathcal{L}_{HSR}$’s ability to extract fine-grained features, particularly benefiting generalization in unseen categories.

\begin{table}[h!]
\centering
\caption{Ablation studies of various components. The symbol $\checkmark$ denotes the inclusion of specific components in the evaluated methods. Performance improvements relative to baseline are indicated by arrows and highlighted in red text.}
\fontsize{6.5}{11}\selectfont 
\begin{tabular}{ccccccc}
\toprule
\rowcolor{gray!20} \textbf{Non-negative} & $\boldsymbol{\mathcal{L}_{\text{NNCE}}}$ & $\boldsymbol{\mathcal{L}_{\text{HSR}}}$ & \textbf{All} & \textbf{Old} & \textbf{New} \\
\midrule
 &  &  & 56.1 & 65.5 & 51.5 \\
\midrule
\rowcolor{gray!10} \checkmark &  &  & 60.4 $\uparrow$ \textcolor{red}{\scriptsize{+4.3}} & 69.0 $\uparrow$ \textcolor{red}{\scriptsize{+3.5}} & 56.1 $\uparrow$ \textcolor{red}{\scriptsize{+4.6}} \\
 & \checkmark &  & 57.9 $\uparrow$ \textcolor{red}{\scriptsize{+1.8}} & 67.6 $\uparrow$ \textcolor{red}{\scriptsize{+2.1}} & 53.1 $\uparrow$ \textcolor{red}{\scriptsize{+1.6}} \\
\rowcolor{gray!10} &  & \checkmark & 62.4 $\uparrow$ \textcolor{red}{\scriptsize{+6.3}} & 72.2 $\uparrow$ \textcolor{red}{\scriptsize{+6.7}} & 57.6 $\uparrow$ \textcolor{red}{\scriptsize{+6.1}} \\
\midrule
\checkmark & \checkmark &  & 60.6 $\uparrow$ \textcolor{red}{\scriptsize{+4.5}} & 68.6 $\uparrow$ \textcolor{red}{\scriptsize{+3.1}} & 60.0 $\uparrow$ \textcolor{red}{\scriptsize{+8.5}} \\
\rowcolor{gray!10} \checkmark &  & \checkmark & 62.7 $\uparrow$ \textcolor{red}{\scriptsize{+6.6}} & 72.4 $\uparrow$ \textcolor{red}{\scriptsize{+6.9}} & 58.0 $\uparrow$ \textcolor{red}{\scriptsize{+6.5}} \\
 & \checkmark & \checkmark & 62.7 $\uparrow$ \textcolor{red}{\scriptsize{+6.6}} & 71.6 $\uparrow$ \textcolor{red}{\scriptsize{+6.1}} & 58.4 $\uparrow$ \textcolor{red}{\scriptsize{+6.9}} \\
\midrule
\rowcolor{gray!20} \checkmark & \checkmark & \checkmark & 65.8 $\uparrow$ \textcolor{red}{\scriptsize{+9.7}} & 73.4 $\uparrow$ \textcolor{red}{\scriptsize{+7.9}} & 62.1 $\uparrow$ \textcolor{red}{\scriptsize{+10.6}} \\
\bottomrule
\end{tabular}
\label{tab:methods_comparison}
\end{table}

\begin{table}[h!]
\centering
\caption{Category discovery accuracy of different pre-trained models using NN-GCD on CUB and Stanford Cars. }
\fontsize{8.0}{11}\selectfont 
\begin{tabular}{l>{\columncolor{Ocean}}ccc>{\columncolor{Ocean}}ccc}
\toprule
\multicolumn{1}{c}{\textbf{Pre-training Method}} & \multicolumn{3}{c}{\textbf{CUB}} & \multicolumn{3}{c}{\textbf{Stanford Cars}} \\
\cmidrule(lr){2-4} \cmidrule(lr){5-7}
& \textbf{All} & \textbf{Old} & \textbf{New} & \textbf{All} & \textbf{Old} & \textbf{New}  \\
\midrule
DINO \cite{caron2021emerging} & 72.4 & 76.3 & 70.5 & 64.8 & 81.9 & 56.6 \\
DINOV2 \cite{oquab2023dinov2}  & 83.1 & 82.5 & 83.4 & 80.3 & 89.9 & 75.7 \\
MAE \cite{he2022masked} & 59.8 & 68.7 & 55.4 & 52.4 & 72.8 & 42.6 \\
iBOT \cite{zhou2021ibot} & 65.9 & 73.2 & 62.3 & 62.5 & 78.1 & 54.9 \\
MocoV3 \cite{chen2021empirical} & 64.7 & 73.7 & 60.2 & 55.5 & 75.6 & 45.8\\
\midrule
SAM \cite{foret2020sharpness} & 65.6 & 71.2 & 62.9 & 48.3 & 66.9 & 39.3\\
DeiT \cite{touvron2021training} & 70.1 & 74.5 & 67.9 & 48.5 & 73.3 & 36.5 \\
DeiT3 \cite{touvron2022deit} & 71.0 & 74.5 & 69.3 & 57.5 & 74.3 & 49.4 \\
CLIP \cite{radford2021learning} & 76.0 & 77.9  & 75.0 &75.1 & 87.4 & 69.1 \\
Supervised-1K \cite{dosovitskiy2020image} & 76.6 & 78.5 & 75.6 & 52.7 & 73.1 & 42.8  \\
Supervised-22K \cite{dosovitskiy2020image} & 72.2 & 74.5 & 71.1 & 46.5 & 68.7 & 35.7  \\
\bottomrule
\end{tabular}
\label{tab: different_pre-trained_models}
\end{table}

Moreover, when examining combinations of two components, we observed significant improvements over the baseline. For instance, combining Non-negative Activated Neurons with $\mathcal{L}_{HSR}$ led to gains of 6.6\% in the "All" category and 6.5\% in the "New" category, showcasing the synergy between feature activation and sparsification. Similarly, the combination of $\mathcal{L}_{\text{NNCE}}$ with $\mathcal{L}_{HSR}$ improved accuracy by 6.6\% in the "All" category and 6.9\% in the "New" category, emphasizing the benefits of balancing category representation and feature sparsification. The results also suggest that $\mathcal{L}_{\text{NNCE}}$ helps balance learning between known and new categories.


When combining all three components, the model achieves its best results, with accuracies of 65.8\%, 73.4\%, and 62.1\% in the "All," "Old," and "New" categories, respectively. These represent significant improvements of 9.7\%, 7.9\%, and 10.6\% over the baseline, demonstrating the complementary effects of feature activation, category balancing, and sparsification. The synergy between these components enhances classification performance, especially in generalizing to new categories.

\begin{figure*}[h] 
    \centering
    \begin{subfigure}[b]{0.30\textwidth}
        \centering
        \includegraphics[width=\textwidth]{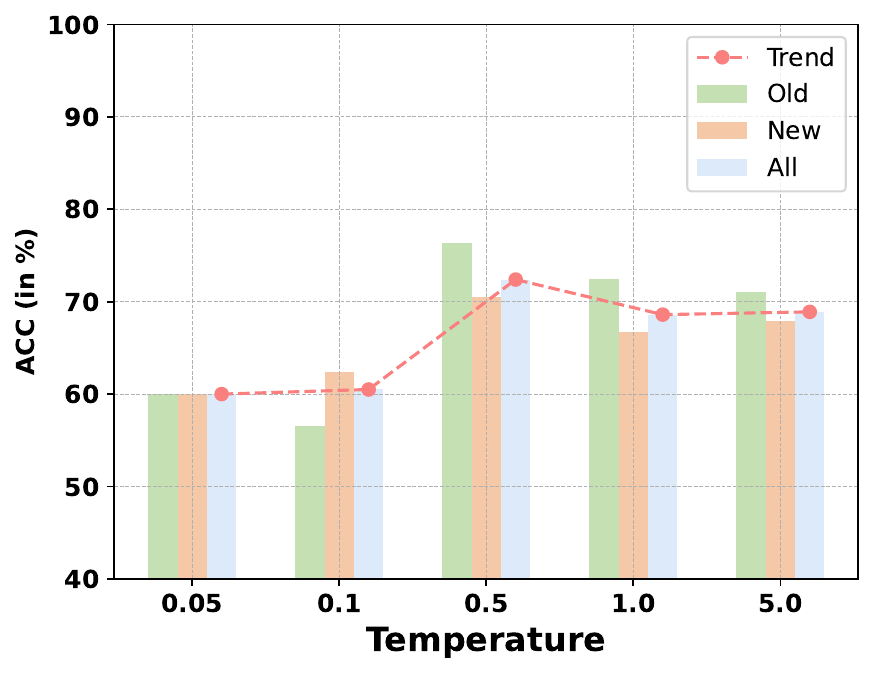}
    \end{subfigure}
    \hspace{0.0001\textwidth}
    \begin{subfigure}[b]{0.30\textwidth}
        \centering
        \includegraphics[width=\textwidth]{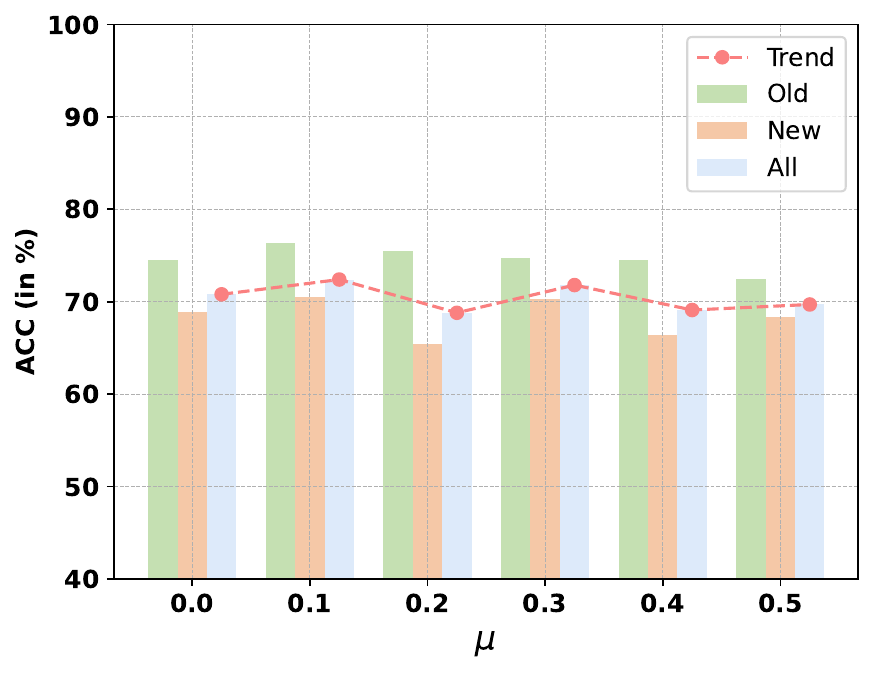}
    \end{subfigure}
    \hspace{0.0001\textwidth}
    \begin{subfigure}[b]{0.30\textwidth}
        \centering
        \includegraphics[width=\textwidth]{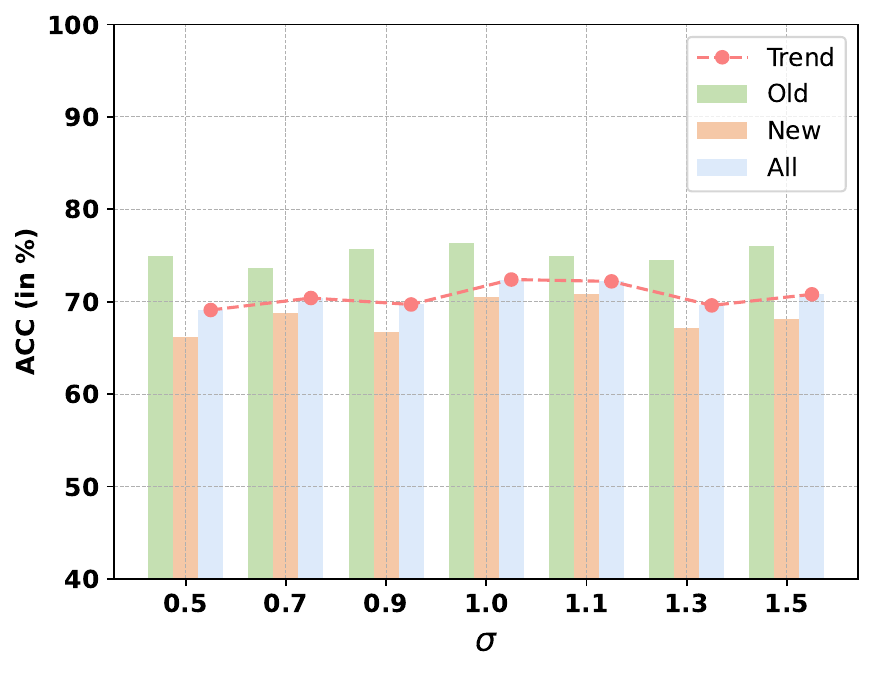}
    \end{subfigure}
    \caption{Results with varying hyperparameters of NMF NCE loss on CUB. From left to right: temperature, \(\mu\) and \(\sigma\).}
    \label{fig: NCE_hyper}
\end{figure*}

\begin{figure*}[h] 
    \centering
    \begin{subfigure}[b]{0.30\textwidth}
        \centering
        \includegraphics[width=\textwidth]{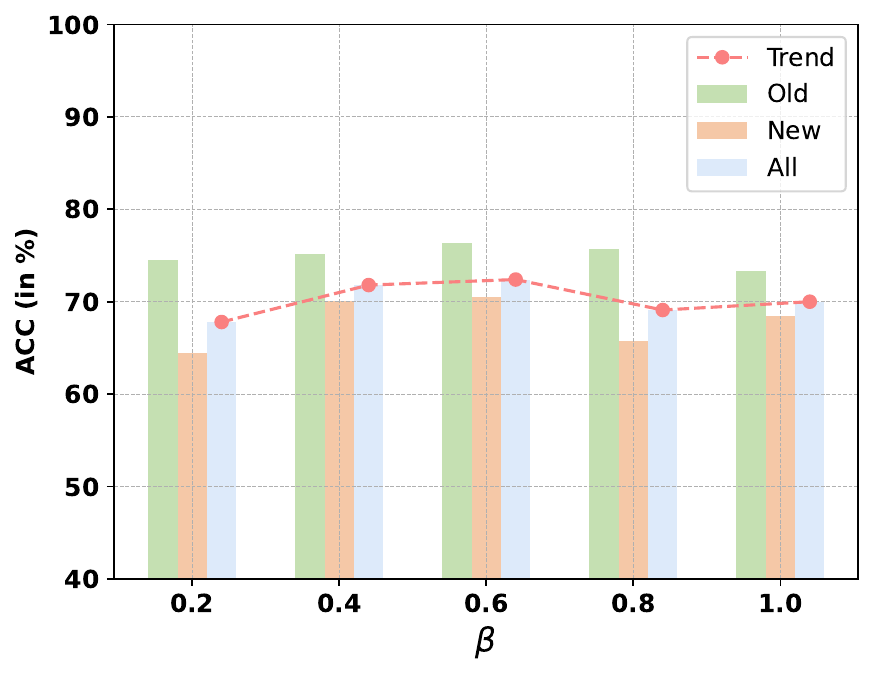}
    \end{subfigure}
    \hspace{0.0001\textwidth}
    \begin{subfigure}[b]{0.30\textwidth}
        \centering
        \includegraphics[width=\textwidth]{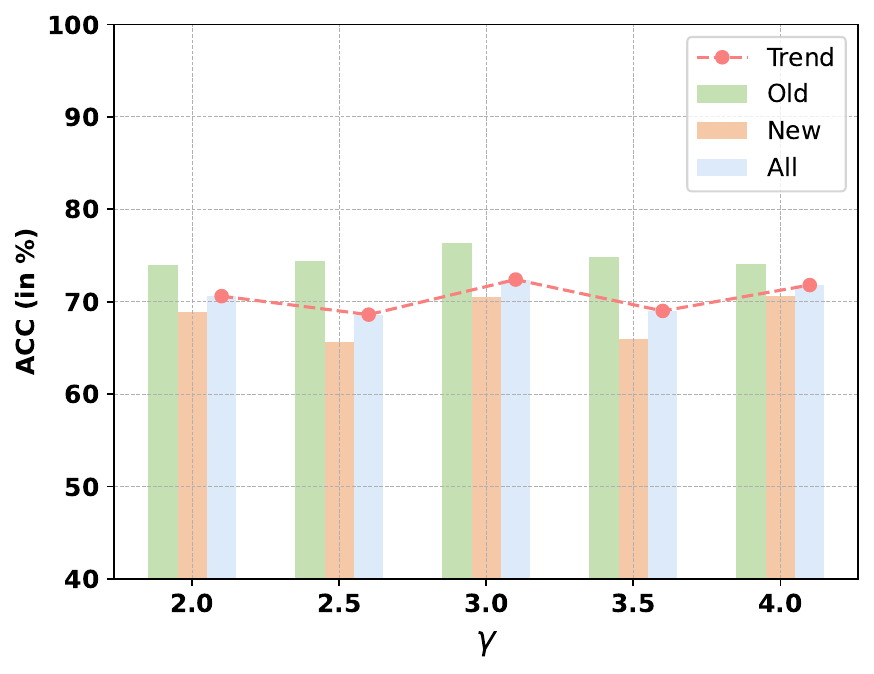}
    \end{subfigure}
    \hspace{0.0001\textwidth}
    \begin{subfigure}[b]{0.30\textwidth}
        \centering
        \includegraphics[width=\textwidth]{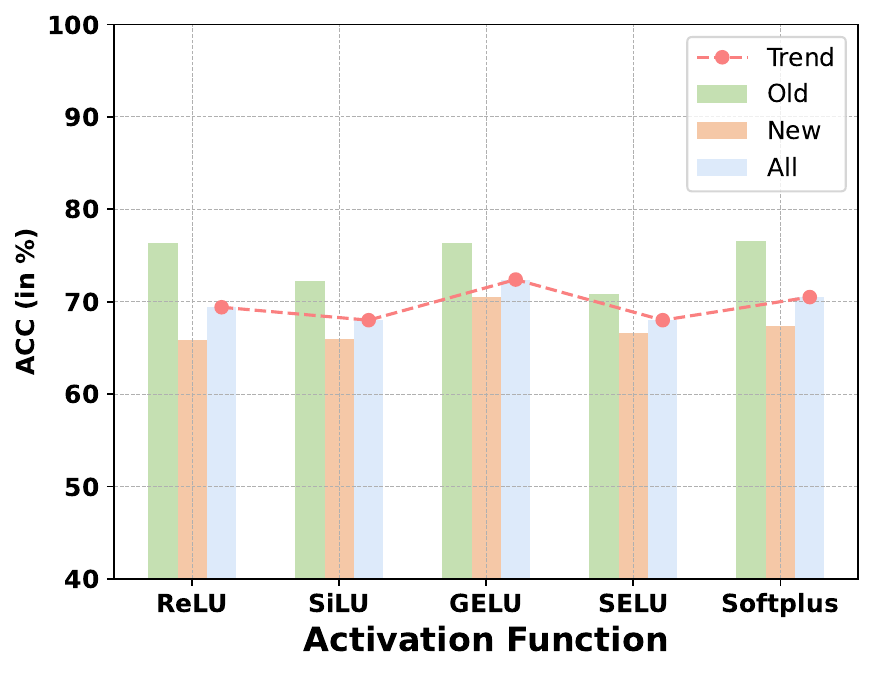}
    \end{subfigure}
    \caption{Results with varying hyperparameters of Hybrid Sparse Regularization and different activation functions on CUB.}
    \label{fig: Reg_hyper}
\end{figure*}

\begin{figure*}[h!] 
    \centering
    \begin{subfigure}[b]{0.8\textwidth}
        \centering
        \includegraphics[width=\textwidth]{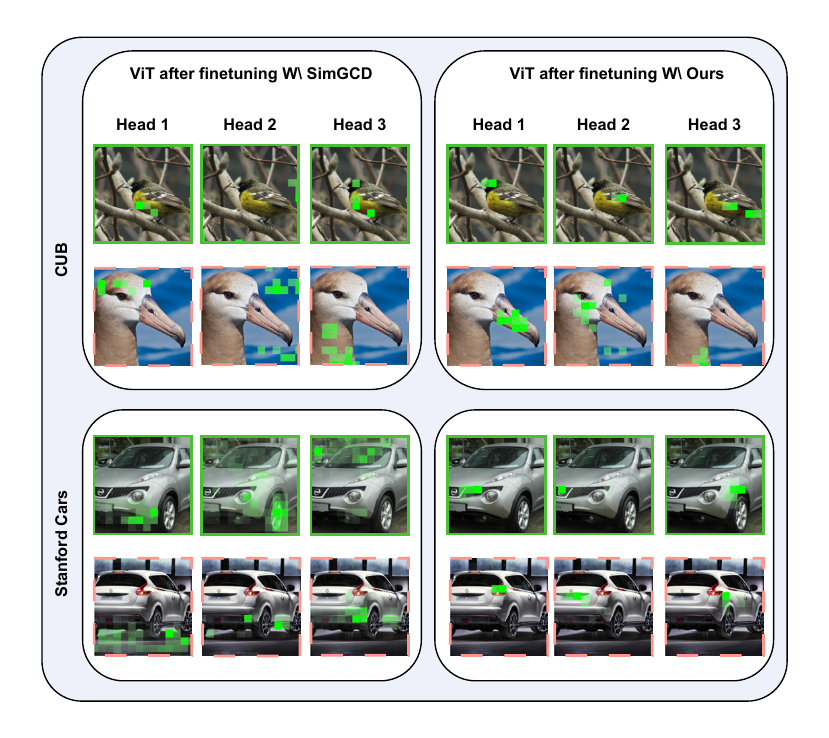}
    \end{subfigure}
    \caption{Attention visualizations.Attention maps for the DINO model fine-tuned with SimGCD (left) and fine-tuned with our approach (right) on the CUB (top) and Stanford Cars (bottom) datasets. For each dataset, we show one row of images from the 'Old' classes (solid green box) and one row of images from the 'New' classes (dashed red box). The heads shown are three randomly selected attention heads from the ViT. Compared to SimGCD, our model better learns to specialize attention heads to different semantically meaningful object parts.}
    \label{fig: visual_attention}
\end{figure*}

\begin{figure*}[h!] 
    \centering
    \begin{subfigure}[b]{0.30\textwidth}
        \centering
        \includegraphics[width=\textwidth]
        {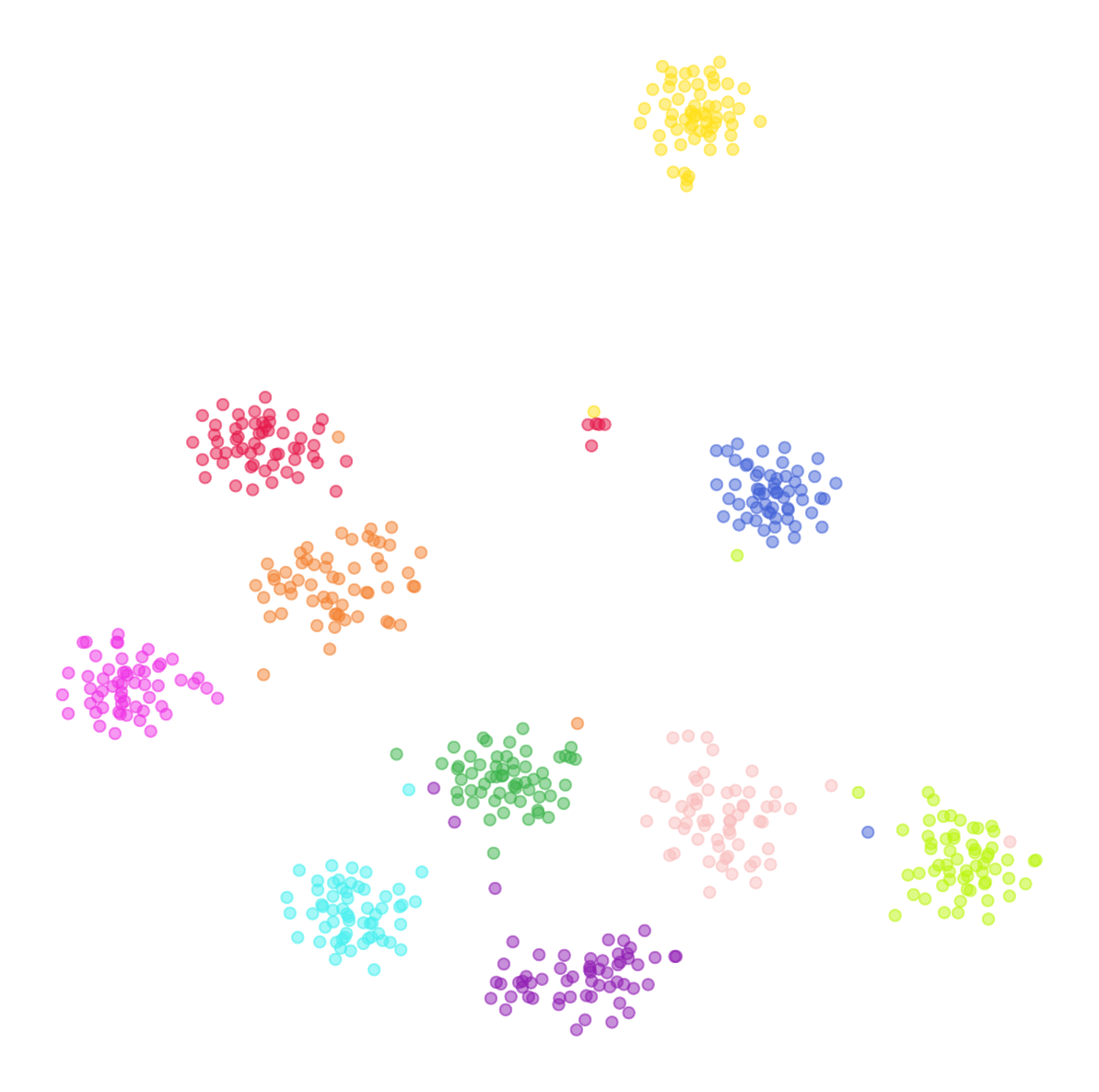}
    \end{subfigure}
    \hspace{0.0001\textwidth}
    \begin{subfigure}[b]{0.30\textwidth}
        \centering
        \includegraphics[width=\textwidth]
        {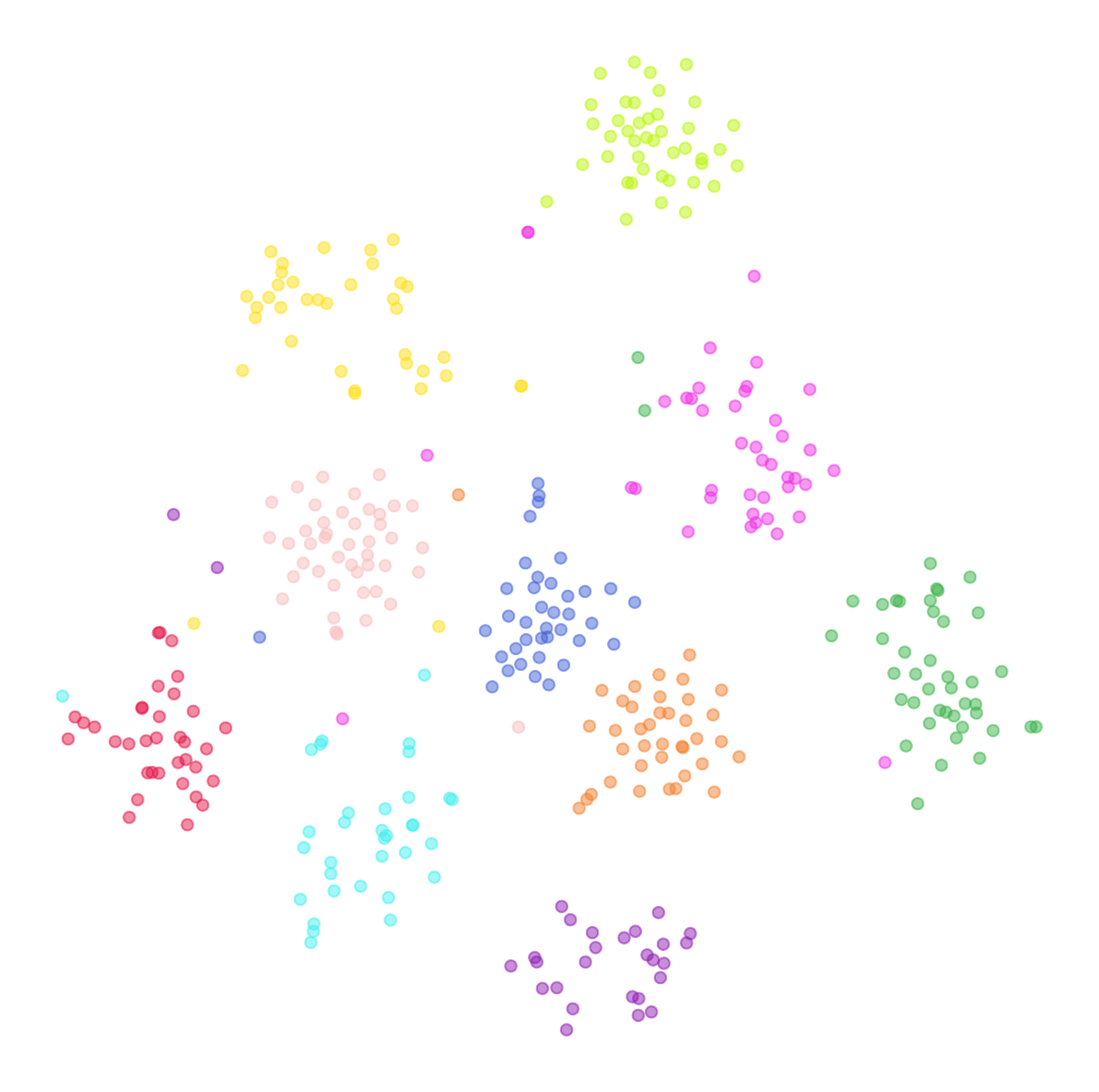}
    \end{subfigure}
    \hspace{0.0001\textwidth}
    \begin{subfigure}[b]{0.30\textwidth}
        \centering
        \includegraphics[width=\textwidth]{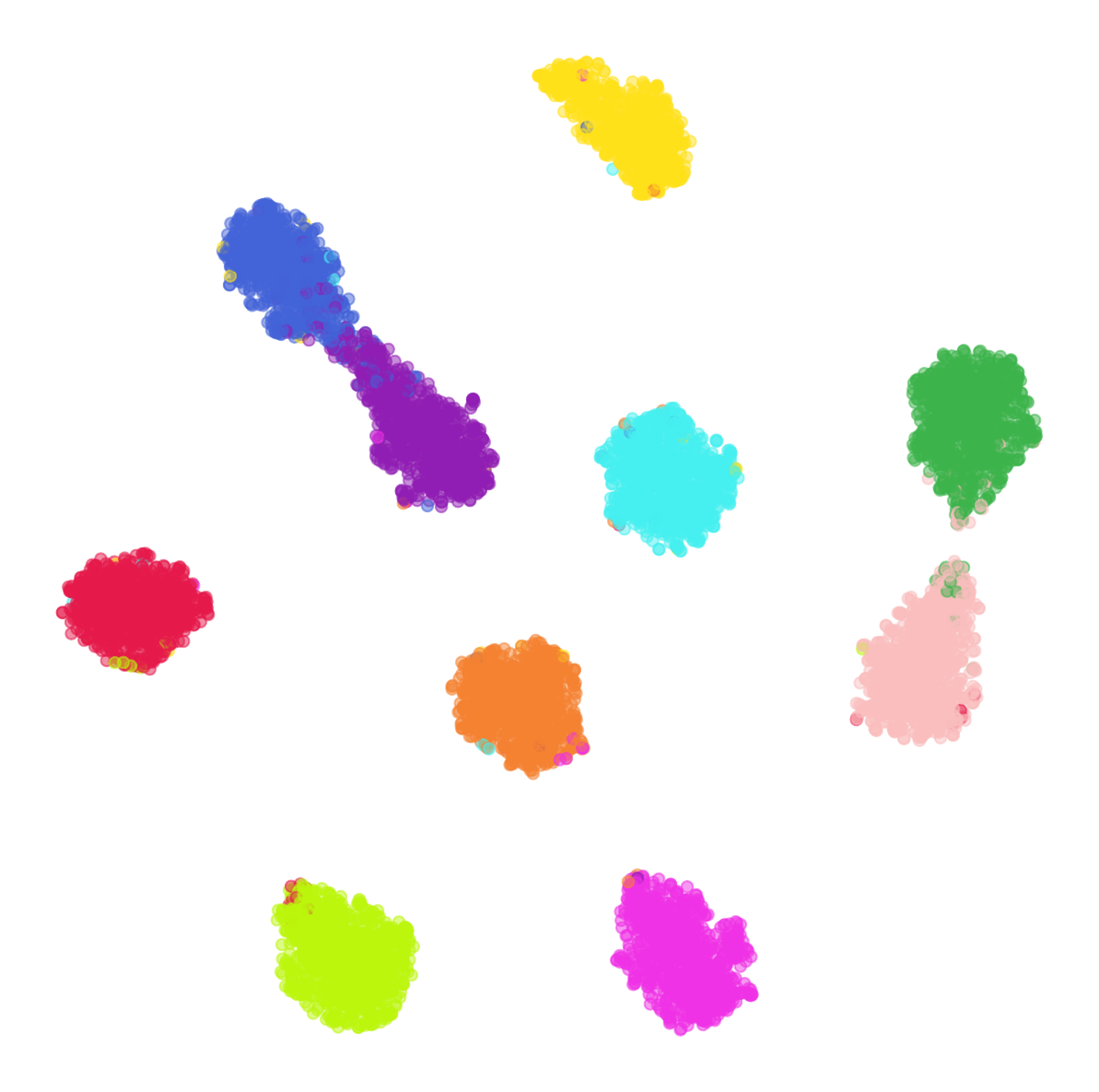}
    \end{subfigure}
    \hspace{0.0001\textwidth}
    \hspace{0.0001\textwidth}
    \caption{T-SNE visualization of CUB, Standford Cars, and CIFAR-10 representations.}
    \label{fig: Visualization}
\end{figure*}


\subsection{Hyper-parameters Study}
Figures~\ref{fig: NCE_hyper} and \ref{fig: Reg_hyper} show the impact of hyperparameter adjustments on the performance of the CUB dataset. For the NMF NCE loss, the optimal hyperparameters are set as a temperature of 0.5, \(\mu\) of 0.1, and \(\sigma\) of 1. These settings ensure balanced contrastive learning between old and new categories, improving overall classification accuracy. The trends in the figures indicate that as the temperature increases, there is a significant improvement in the performance of the "New" category, especially around a temperature of 0.5, highlighting the importance of appropriate contrastive scaling for recognizing new categories. However, changes in \(\mu\) and \(\sigma\) have a more modest impact, with optimal values observed at \(\mu = 0.1\) and \(\sigma = 1\), showing that they fine-tune category balance without causing large performance fluctuations.

For Hybrid Sparse Regularization and the Non-negative Activated Neurons module, the optimal value for \(\beta\) is 0.6, and although \(\gamma\) is shown as 3 in the figures, our code uses \(1e^{-5}\) as a scaling factor. The actual optimal value is \(3e^{-5}\), and this scaling factor is chosen about the size of the model's parameters. Larger models typically require smaller scaling factors to ensure effective regularization without overly penalizing the learning process. GELU was selected as the best activation function. The figures demonstrate that both GELU and SiLU significantly improve performance in both "Old" and "New" categories, with GELU providing slightly better generalization to unseen categories. The performance trends indicate that increasing \(\beta\) and \(\gamma\) helps capture finer-grained features, with \(\gamma = 3e^{-5}\) providing consistent improvements across all categories. This suggests that the combination of appropriate sparsity regularization and activation functions is crucial for enhancing fine-grained feature extraction, especially for recognizing new categories.

As shown in Tab.~\ref{tab: different_pre-trained_models}, the performance of different pre-training methods on the CUB and Stanford Cars datasets highlights the significant impact that pre-training weights have on fine-grained classification tasks. DINO and DINOv2 stand out, particularly DINOv2, which achieves 83.1\% accuracy in the "All" category on the CUB dataset and 80.3\% on the Stanford Cars dataset, demonstrating strong generalization capabilities, especially for new category recognition.

As shown in Tab.~\ref{tab: different_pre-trained_models}, regardless of whether a self-supervised or supervised pre-training model is used, our method consistently maintains robust performance across various pre-trained weight initializations. For example, even when using MAE \cite{he2022masked} or iBOT \cite{zhou2021ibot} pre-training, the method still achieves high accuracy on both datasets. This demonstrates that our approach is adaptable and versatile, not relying on a specific pre-training model, and can maintain strong classification performance under a variety of initialization conditions.

These results highlight the robustness of our method across diverse pre-training settings, further validating its effectiveness and broad applicability.

\subsection{Visualization}
Figure \ref{fig: visual_attention} shows the attention maps of the DINO model fine-tuned with SimGCD (left) and our approach (right) on the CUB (top) and Stanford Cars (bottom) datasets. It is clear that our approach more effectively focuses attention heads on semantically meaningful object parts compared to SimGCD. This demonstrates that the introduction of sparsity and Non-negative Activated Neurons enhances both model performance and interpretability. The non-negativity constraint ensures that features represent meaningful components, leading to a more interpretable factorization. Moreover, the non-negativity constraint encourages an additive parts-based representation by preventing the cancellation of components, thereby enhancing the interpretability of the attention heads.

Figure \ref{fig: Visualization} presents T-SNE visualizations of representations from the CUB, Stanford Cars, and CIFAR-10 datasets, with ten classes selected from each dataset. These visualizations highlight the model's strong representational power, characterized by large inter-class distances and small intra-class distances, indicating that it has achieved an optimal $\bar{A}^*$.


\section{Conclusion}
In this paper, we present NN-GCD, a novel framework tackling GCD complexities. Through the meticulous optimization of the co-occurrence matrix $\bar{A}$, and by harnessing the synergistic theoretical underpinnings of K-means clustering, SNMF, and NCL, we have devised a Non-negative Activated Neurons mechanism. This innovative construct, further reinforced by the NMF NCE loss coupled with Hybrid Sparsity Regularization, facilitates an unprecedented level of intra-class compactness and inter-class differentiation, culminating in the derivation of the optimal $\bar{A}^*$. Rigorous empirical evaluations affirm the NN-GCD framework's superior efficacy, solidifying its status as a vanguard methodology in the realm of advanced category discovery.

{
    \small
    \bibliographystyle{IEEEtran}
    \bibliography{main}
}

\vfill
\end{document}